\documentclass{article}

\usepackage[preprint,nonatbib]{Preprint_MMMS}
\usepackage{amssymb}
\usepackage{amsthm}
\usepackage{amsfonts}
\usepackage{amsmath}

\usepackage[utf8]{inputenc} 
\usepackage[T1]{fontenc}    
\usepackage{hyperref}       
\usepackage{url}            
\usepackage{booktabs}       
\usepackage{amsfonts}       
\usepackage{nicefrac}       
\usepackage{microtype}      
\usepackage{xcolor}         
\usepackage{amsmath}
\usepackage{amssymb}
\usepackage{amsthm}
\usepackage{natbib}

\newtheorem{theorem}{Theorem}[section]

\newtheorem{proposition}[theorem]{Proposition}
\newtheorem{lemma}[theorem]{Lemma}

\newtheorem{corollary}[theorem]{Corollary}

\title{Generalization of Hamiltonian algorithms}

\author{%
  Andreas Maurer\\
  Istituto Italiano di Tecnologia, CSML, 16163 Genoa, Italy\\
  \texttt{am@andreas-maurer.eu} \\
}

\begin{document}

\maketitle

\begin{abstract}
  The paper proves generalization results for a class of stochastic learning algorithms. The method applies whenever the algorithm generates an absolutely continuous distribution relative to some a-priori measure and the Radon Nikodym derivative has subgaussian concentration. Applications are bounds for the Gibbs algorithm and randomizations of stable deterministic algorithms as well as PAC-Bayesian bounds with data-dependent priors.
\end{abstract}

\section{Introduction}

A stochastic learning algorithm $Q$ takes as input a sample $\mathbf{X}%
=\left( X_{1},...,X_{n}\right) \in \mathcal{X}^{n}$, drawn from a
distribution $\mu $ on a space $\mathcal{X}$ of data, and outputs a
probability measure $Q_{\mathbf{X}}$ on a loss-class $\mathcal{H}$ of
functions $h:\mathcal{X}\mapsto \left[ 0,\infty \right) $. A key problem in
the study of these algorithms is to bound the generalization gap%
\begin{equation}
	\Delta \left( h,\mathbf{X}\right) =\mathbb{E}\left[ h\left( X\right) \right]
	-\frac{1}{n}\sum_{i=1}^{n}h\left( X_{i}\right)  \label{Generalization gap}
\end{equation}%
between the expected and the empirical loss of a hypothesis $h$ drawn from $%
Q_{\mathbf{X}}$. Here we want to generate $h$ \textit{only once} and seek
guarantees with high probability as $\mathbf{X}\sim \mu ^{n}$ \textit{and} $%
h\sim Q_{\mathbf{X}}$. Alternatively one might want a bound on the
expectation $\mathbb{E}_{h\sim Q_{\mathbf{X}}}\left[ \Delta \left( h,\mathbf{%
	X}\right) \right] $ with high probability in $\mathbf{X}\sim \mu ^{n}$,
corresponding to the use of a stochastic hypothesis, where a new $h\sim Q_{%
	\mathbf{X}}$ is generated for every test point. We concentrate on the former
question, but many of the techniques presented also apply to the latter,
often easier problem.

From Markov's inequality it follows that for $\lambda $, $\delta >0$ with
probability at least $1-\delta $ as $\mathbf{X}\sim \mu ^{n}$ \textit{and} $%
h\sim Q_{\mathbf{X}}$%
\begin{equation}
	\Delta \left( h,\mathbf{X}\right) \leq \frac{\ln \mathbb{E}_{\mathbf{X}}%
		\left[ \mathbb{E}_{h\sim Q_{\mathbf{X}}}\left[ e^{\lambda \Delta \left( h,%
			\mathbf{X}\right) }\right] \right] +\ln \left( 1/\delta \right) }{\lambda },
	\label{General stochastic algorithm bound}
\end{equation}%
which suggests to bound the log-moment generating function $\ln \mathbb{E}_{%
	\mathbf{X}}\left[ \mathbb{E}_{h\sim Q_{\mathbf{X}}}\left[ \exp \left(
\lambda \Delta \left( h,\mathbf{X}\right) \right) \right] \right] $. With
such a bound at hand one can optimize $\lambda $ to establish generalization
of the algorithm $Q:\mathbf{X}\mapsto Q_{\mathbf{X}}$.

Inequality (\ref{General stochastic algorithm bound}) is relevant to
stochastic algorithms in general, and in particular to the Gibbs-algorithm,
where $dQ_{\mathbf{X}}\left( h\right) \propto \exp \left( -\left( \beta
/n\right) \sum h\left( X_{i}\right) \right) d\pi \left( h\right) $ for some
inverse temperature parameter $\beta $ and some nonnegative a priori measure 
$\pi $ on $\mathcal{H}$. The Gibbs algorithm has its origins in statistical
mechanics (\cite{gibbs1902elementary}). In the context of machine learning
it can be viewed as a randomized version of empirical risk minimization, to
which it converges as $\beta \rightarrow \infty $, whenever $\pi $ has full
support. The distribution, often called Gibbs posterior (\cite{catoni2007pac}%
), is a minimizer of the PAC-Bayesian bounds (\cite{mcallester1999pac}). It
is also the limiting distribution of stochastic gradient Langevin dynamics (%
\cite{raginsky2017non}) under rather general conditions. Generalization
bounds in expectation are given by \cite{raginsky2017non}, \cite%
{kuzborskij2019distribution}, most recently by \cite{aminian2021exact}.
Bounds in probability are given by \cite{lever2013tighter}, implicitly by 
\cite{dziugaite2018data}, and in \cite{rivasplata2020pac} following the
method of \cite{kuzborskij2019distribution}. There is also a bound by \cite%
{aminian2023information}, improving on the one in (\cite{lever2013tighter}).

Bounding $\ln \mathbb{E}_{\mathbf{X}}\left[ \mathbb{E}_{h\sim Q_{\mathbf{X}}}%
\left[ \exp \left( \lambda \Delta \left( h,\mathbf{X}\right) \right) \right] %
\right] $ is also the vehicle (and principal technical obstacle) to prove
PAC-Bayesian bounds with data-dependent prior $Q_{\mathbf{X}}$, as pointed
out by \cite{rivasplata2020pac} (Theorem 1). Such bounds with \textit{%
	data-independent} prior $Q_{\mathbf{X}}=Q$, have an over twenty year old
tradition in learning theory, starting with the seminal work of McAllester (%
\cite{mcallester1999pac}), Langford and Seeger (\cite{langford2001bounds}, 
\cite{seeger2002pac}), see also \cite{guedj2019primer}. If the prior is
data-independent, the two expectations in $\ln \mathbb{E}_{\mathbf{X}}\left[ 
\mathbb{E}_{h\sim Q}\left[ \exp \left( \lambda \Delta \left( h,\mathbf{X}%
\right) \right) \right] \right] $ can be exchanged, which reduces the
analysis to classical Chernoff- or Hoeffding-inequalities. But a dominant
term in these bounds, the KL-divergence $KL\left( P,Q\right) :=\mathbb{E}%
_{h\sim P}\left[ \ln \left( dP/dQ\right) \left( h\right) \right] $, will be
large unless $P$ is well aligned to $Q$, so the prior $Q$ should already put
more weight on good hypotheses with small loss. This motivates the use of
distribution-dependent priors, and as the distribution is unknown, one is
led to think about data-dependent priors. Catoni already considers the
data-dependent Gibbs distribution as a prior in a derivation departing from
the distribution-dependent Gibbs measure $\propto \exp \left( -\beta \mathbb{%
	E}_{X\sim \mu }\left[ h\left( X\right) \right] \right) $ (Lemma 6.2 in \cite%
{catoni2003pac}). \cite{dziugaite2017computing} used a Gaussian prior with
data-dependent width and minimized the PAC-Bayes bound for a Gaussian
posterior on a multi-layer neural network, obtaining a good classifier
accompanied by a non-vacuous bound. This significant advance raised interest
in PAC-Bayes with data-dependent priors. The same authors introduced a
method to control data-dependent priors based on differential privacy (\cite%
{dziugaite2018data}). Recently \cite{perez2021tighter} used Gaussian and
Laplace priors, whose means were trained directly from one part of the
sample, the remaining part being used to evaluate the PAC-Bayes bound. These
developments further motivate the search for in-sample bounds on the
log-moment generating function appearing above in (\ref{General stochastic
	algorithm bound}).

We make the following contributions:

\begin{itemize}
	\item An economical and general method to bound $\ln \mathbb{E}_{\mathbf{X}}%
	\left[ \mathbb{E}_{h\sim Q_{\mathbf{X}}}\left[ \exp \left( \lambda \Delta
	\left( h,\mathbf{X}\right) \right) \right] \right] $ whenever the logarithm
	of the density of $Q_{\mathbf{X}}$ concentrates exponentially about its
	mean. In particular, whenever $Q_{\mathbf{X}}$ has the Hamiltonian form $dQ_{%
		\mathbf{X}}\left( h\right) $ $\propto \exp \left( H\left( h,\mathbf{X}%
	\right) \right) d\pi \left( h\right) $, then it is sufficient that the
	Hamiltonian $H$ satisfies a bounded difference condition. The method also
	extends to the case, when $H$ is only sub-Gaussian in its arguments.
	
	\item Applications to the Gibbs algorithm yielding competitive
	generalization guarantees, both for bounded and sub-Gaussian losses. Despite
	its simplicity and generality the method improves over existing results on
	this well studied problem, removing unnecessary logarithmic factors and
	various superfluous terms.
	
	\item Generalization guarantees for hypotheses sampled once from stochastic
	kernels centered at the output of uniformly stable algorithms, considerably
	strengthening a previous result of \cite{rivasplata2018pac}.
\end{itemize}

\section{Notation and Preliminaries}

For $m\in \mathbb{N}$, we write $\left[ m\right] :=\left\{ 1,...,m\right\} $%
. Random variables are written in upper case letters, like $X,Y,\mathbf{X}$
etc and primes indicate iid copies, like $X^{\prime }$, $X^{\prime \prime }$%
, $\mathbf{X}^{\prime }$ etc. Vectors are bold like $\mathbf{x}$,$\mathbf{X}$%
, etc. Throughout $\mathcal{X}$ is a measurable space of data with a
probability measure $\mu $, and $\mathbf{X}$ will always be the iid vector $%
\mathbf{X}=\left( X_{1},...,X_{n}\right) \sim \mu ^{n}$ and $X\sim \mu $. $%
\mathcal{H}$ is a measurable space of measurable functions $h:\mathcal{%
	X\rightarrow }\left[ 0,\infty \right) $, and there is a nonnegative a priori
measure $\pi $ on $\mathcal{H}$. The measure $\pi $ need not be a
probability measure, it could be Lebesgue measure on the space $\mathbb{R}%
^{d}$ of parametrizations of $\mathcal{H}$. Averages over $\pi $ will be
written as integrals. $\mathcal{P}\left( \mathcal{H}\right) $ is the set of
probability measures on $\mathcal{H}$. Unless otherwise specified $\mathbb{E}
$ denotes the expectation in $X\sim \mu $ or $\mathbf{X}\sim \mu ^{n}$. All
functions on $\mathcal{H\times X}^{n}$ appearing in this paper are assumed
to have exponential moments of all orders, with respect to both arguments.

For $\mathbf{x}\in \mathcal{X}^{n}$, $k\in \left\{ 1,...,n\right\} $ and $%
y,y^{\prime }\in \mathcal{X}$ we define the substitution operator $S_{y}^{k}$
acting on $\mathcal{X}^{n}$ and the partial difference operator $%
D_{y,y^{\prime }}^{k}$ acting on functions $f:\mathcal{X}^{n}\rightarrow 
\mathbb{R}
$ by 
\begin{equation*}
	S_{y}^{k}\mathbf{x}=\left( x_{1},...,x_{k-1},y,x_{k+1},...,x_{n}\right) 
	\text{ and }D_{y,y^{\prime }}^{k}f\left( \mathbf{x}\right) =f\left( S_{y}^{k}%
	\mathbf{x}\right) -f\left( S_{y^{\prime }}^{k}\mathbf{x}\right) \text{.}
\end{equation*}%
$D_{y,y^{\prime }}^{k}$ always refers to the second argument for functions
on $\mathcal{H\times X}^{n}$. The generalization gap $\Delta \left( h,%
\mathbf{X}\right) $ is defined as in (\ref{Generalization gap}). Sometimes
we write $L\left( h\right) =\mathbb{E}\left[ h\left( X\right) \right] $ and $%
\hat{L}\left( h,\mathbf{X}\right) =\left( 1/n\right) \sum_{i=1}^{n}h\left(
X_{i}\right) $, so that $\Delta \left( h,\mathbf{X}\right) =L\left( h\right)
-\hat{L}\left( h,\mathbf{X}\right) $. A table of notation is provided in
Appendix \ref{table of notation}.

\subsection{Hamiltonian algorithms}

A stochastic algorithm $Q:\mathbf{x}\in \mathcal{X}^{n}\mapsto Q_{\mathbf{x}%
}\in \mathcal{P}\left( \mathcal{H}\right) $ will be called absolutely
continuous, if $Q_{\mathbf{x}}$ is absolutely continuous with respect to $%
\pi $ for every $\mathbf{x}\in \mathcal{X}^{n}$ and vice versa. We will only
consider absolutely continuous algorithms in the sequel. A real function $H$
on $\mathcal{H\times X}^{n}$ is called a Hamiltonian for $Q$ (a term taken
from statistical physics), if for all $h\in \mathcal{H}$ and all $\mathbf{x}%
\in \mathcal{X}^{n}$%
\begin{equation*}
	dQ_{\mathbf{x}}\left( h\right) =\frac{e^{H\left( h,\mathbf{x}\right) }d\pi
		\left( h\right) }{Z\left( \mathbf{x}\right) }\text{ with }Z\left( \mathbf{x}%
	\right) =\int_{\mathcal{H}}e^{H\left( h,\mathbf{x}\right) }d\pi \left(
	h\right) .
\end{equation*}%
The normalizing function $Z$ is called the partition function$.$ Every
absolutely continuous $Q$ has the canonical Hamiltonian $H_{Q}\left( h,%
\mathbf{x}\right) =\ln \left( \left( dQ_{\mathbf{x}}/d\pi \right) \left(
h\right) \right) $ (logarithm of the Radon Nikodym derivative) with
partition function $Z\equiv 1$, but adding any function $\zeta :\mathcal{X}%
^{n}\rightarrow \mathbb{R}$ to $H_{Q}$ will give a Hamiltonian for the same
algorithm with partition function $Z\left( \mathbf{x}\right) =\exp \left(
\zeta \left( \mathbf{x}\right) \right) $. In practice $Q$ is often defined
by specifying some Hamiltonian $H$, so $H_{Q}\left( h,\mathbf{x}\right)
=H\left( h,\mathbf{x}\right) -\ln Z\left( \mathbf{x}\right) $ in general. If 
$\pi $ is a probability measure, then $\mathbb{E}_{h\sim Q_{\mathbf{x}}}%
\left[ H_{Q}\left( h,\mathbf{x}\right) \right] $ is the KL-divergence $%
KL\left( Q_{\mathbf{x}},\pi \right) $.

A Hamiltonian for the Gibbs algorithm at inverse temperature $\beta $ is 
\begin{equation*}
	H\left( h,\mathbf{x}\right) =-\beta \hat{L}\left( h,\mathbf{x}\right) =-%
	\frac{\beta }{n}\sum_{i=1}^{n}h\left( x_{i}\right) ,
\end{equation*}%
putting larger weights on hypotheses with small empirical loss. This is the
simplest case covered by our proposed method. If there is a computational
cost associated with each $h$, it may be included to promote hypotheses with
faster execution. We could also add a negative multiple of $\sum_{i<j}\left(
h\left( x_{i}\right) -h\left( x_{j}\right) \right) ^{2}$, so as to encourage
hypotheses with small empirical variance. Monte Carlo methods, such as the
Metropolis-Hastings algorithm, can be used to sample from such
distributions. Often these are slow to converge, which underlines the
practical importance of using a single hypothesis generated only once.

If $\mathcal{H}$ is parametrized by $\mathbb{R}^{d}$ one may also first
compute a vector $A\left( \mathbf{x}\right) \in \mathbb{R}^{d}$ with some
algorithm $A$ and then sample from an absolutely continuous stochastic
kernel $\kappa $ centered at $A\left( \mathbf{x}\right) $, so a Hamiltonian
is $\ln \kappa \left( h-A\left( \mathbf{x}\right) \right) $. In one concrete
version the kernel is an isotropic gaussian, and $H\left( h,\mathbf{x}%
\right) =-\left\Vert h-A\left( \mathbf{x}\right) \right\Vert ^{2}/\left(
2\sigma ^{2}\right) $. Generalization of these methods is discussed in
Section \ref{Section Randomization of stable algorithms}.

If $Q^{\left( 1\right) }$ and $Q^{\left( 2\right) }$ are absolutely
continuous stochastic algorithms with Hamiltonians $H_{1}$ and $H_{2}$
respectively, then an elementary calculation shows that $H_{1}+H_{2}$ is a
Hamiltonian for the algorithm $Q$ obtained by sampling from $Q^{\left(
	1\right) }$ with the measure $\pi $ replaced by $Q^{\left( 2\right) }$. In
this way Hamiltonian algorithms of different type can be combined.

\section{Main results\label{Section Main results}}

Let $Q$ be an absolutely continuous stochastic algorithm, $F:\mathcal{H}%
\times \mathcal{X}^{n}\rightarrow \mathbb{R}$ some function and define%
\begin{equation*}
	\mathbb{\psi }_{F}\left( h\right) :=\ln \mathbb{E}\left[ e^{F\left( h,%
		\mathbf{X}\right) +H_{Q}\left( h,\mathbf{X}\right) -\mathbb{E}\left[
		H_{Q}\left( h,\mathbf{X}^{\prime }\right) \right] }\right] .
\end{equation*}%
Our method is based on the following proposition.

\begin{proposition}
	\label{Proposition Main}With $Q$, $F$ and $\psi $ as above
	
	(i) $\ln \mathbb{E}_{\mathbf{X}\sim \mu ^{n}}\mathbb{E}_{h\sim Q_{\mathbf{X}%
	}}\left[ e^{F\left( h,\mathbf{X}\right) }\right] \leq \sup_{h\in \mathcal{H}%
	}\psi _{F}\left( h\right) .$
	
	(ii) Let $\delta >0$. Then with probability at least $1-\delta $ in $\mathbf{%
		X}\sim \mu ^{n}$ and $h\sim Q_{\mathbf{X}}$ we have%
	\begin{equation*}
		F\left( h,\mathbf{X}\right) \leq \sup_{h\in \mathcal{H}}\psi _{F}\left(
		h\right) +\ln \left( 1/\delta \right) .
	\end{equation*}
	
	(iii) Let $\delta >0$. Then with probability at least $1-\delta $ in $%
	\mathbf{X}\sim \mu ^{n}$ we have%
	\begin{equation*}
		\mathbb{E}_{h\sim Q_{\mathbf{X}}}\left[ F\left( h,\mathbf{X}\right) \right]
		\leq \sup_{h\in \mathcal{H}}\psi _{F}\left( h\right) +\ln \left( 1/\delta
		\right) .
	\end{equation*}
\end{proposition}

\begin{proof}
	(i) With Jensen's inequality%
	\begin{align*}
		& \ln \mathbb{E}_{\mathbf{X}\sim \mu ^{n}}\mathbb{E}_{h\sim Q_{\mathbf{X}}}%
		\left[ e^{F\left( h,\mathbf{X}\right) }\right] =\ln \mathbb{E}_{\mathbf{X}%
			\sim \mu ^{n}}\left[ \int_{\mathcal{H}}e^{F\left( h,\mathbf{X}\right)
			+H_{Q}\left( h,\mathbf{X}\right) }d\pi \left( h\right) \right]  \\
		& =\ln \int_{\mathcal{H}}\mathbb{E}_{\mathbf{X}\sim \mu ^{n}}\left[
		e^{F\left( h,\mathbf{X}\right) +H_{Q}\left( h,\mathbf{X}\right) -\mathbb{E}%
			\left[ H_{Q}\left( h,\mathbf{X}^{\prime }\right) \right] }\right] e^{\mathbb{%
				E}\left[ H_{Q}\left( h,\mathbf{X}^{\prime }\right) \right] }d\pi \left(
		h\right)  \\
		& =\ln \int_{\mathcal{H}}e^{\mathbb{\psi }_{F}\left( h\right) }e^{\mathbb{E}%
			\left[ H_{Q}\left( h,\mathbf{X}^{\prime }\right) \right] }d\pi \left(
		h\right)  \\
		& \leq \ln \int_{\mathcal{H}}\mathbb{E}\left[ e^{\psi _{F}\left( h\right)
		}e^{H_{Q}\left( h,\mathbf{X}^{\prime }\right) }\right] d\pi \left( h\right)
		=\ln \mathbb{E}\left[ \int_{\mathcal{H}}e^{\psi _{F}\left( h\right)
		}e^{H_{Q}\left( h,\mathbf{X}^{\prime }\right) }d\pi \left( h\right) \right] 
		\\
		& =\ln \mathbb{E}_{\mathbf{X}\sim \mu ^{n}}\mathbb{E}_{h\sim Q_{\mathbf{X}}}%
		\left[ e^{\psi _{F}\left( h\right) }\right] \leq \sup_{h\in \mathcal{H}}\psi
		_{F}\left( h\right) .
	\end{align*}%
	(ii) then follows from Markov's inequality (Section \ref{Section Markovs
		inequality}). (iii) follows also by Markov's inequality, since $\ln \mathbb{E%
	}_{\mathbf{X}\sim \mu ^{n}}\left[ e^{\mathbb{E}_{h\sim Q_{\mathbf{X}}}\left[
		F\left( h,\mathbf{X}\right) \right] }\right] \leq \ln \mathbb{E}_{\mathbf{X}%
		\sim \mu ^{n}}\mathbb{E}_{h\sim Q_{\mathbf{X}}}\left[ e^{F\left( h,\mathbf{X}%
		\right) }\right] $, by Jensen's inequality.
\end{proof}

To see the point of this lemma let $F\left( h,\mathbf{X}\right) =\lambda
\Delta \left( h,\mathbf{X}\right) $. Since $\Delta \left( h,\mathbf{X}%
\right) $ is centered, $\psi _{\lambda \Delta }$ is of the form $\ln \mathbb{%
	E}_{\mathbf{X}\sim \mu ^{n}}\left[ e^{f\left( \mathbf{X}\right) -\mathbb{E}%
	\left[ f\left( \mathbf{X}^{\prime }\right) \right] }\right] $. Many
concentration inequalities in the literature (\cite{McDiarmid98}, \cite%
{Boucheron13}) follow the classical ideas of Bernstein and Chernoff and are
derived from bounds on such moment generating functions. If $F\left( h,%
\mathbf{X}\right) $ is not centered, we can use H\"{o}lder's or the
Cauchy-Schwarz inequality to separate the contributions which $F$ and $H_{Q}$
make to $\psi _{F}$. The $F$-contribution can be treated separately and the
contribution of $H_{Q}$ can again be treated with the methods of
concentration inequalities.

The last conclusion of the lemma is to show that we can always get bounds in
expectation from bounds on the exponential moment. In the sequel we only
state the stronger un-expected or "disintegrated" results.

Typically the exponential moment bounds for functions of independent
variables depend on the function's stability with respect to changes in its
arguments. In Section \ref{Section Randomization of stable algorithms} a
more advanced concentration inequality will be used, but all other results
below depend only on the following classical exponential moment bounds. Most
of them can be found in \cite{McDiarmid98}, but since the results there are
formulated as deviation inequalities, a proof is given in the appendix,
Section \ref{Section Proof of proposition Martingale}.

\begin{proposition}
	\label{Proposition my Martingale}Let $X,X_{1},...,X_{n}$ be iid random
	variables with values in $\mathcal{X}$, $\mathbf{X}=\left(
	X_{1},...,X_{n}\right) $ and $f:\mathcal{X}^{n}\rightarrow \mathbb{R}$
	measurable.
	
	(i) If $f$ is such that for all $k\in \left[ n\right] $, $\mathbf{x}\in 
	\mathcal{X}^{n}$ we have $\mathbb{E}_{X}\left[ e^{f\left( S_{X}^{k}\mathbf{x}%
		\right) -\mathbb{E}_{X^{\prime }}\left[ f\left( S_{X^{\prime }}^{k}\mathbf{x}%
		\right) \right] }\right] \leq e^{r^{2}}$, then $\mathbb{E}\left[ e^{f\left( 
		\mathbf{X}\right) -\mathbb{E}\left[ f\left( \mathbf{X}^{\prime }\right) %
		\right] }\right] \leq e^{nr^{2}}.$
	
	(ii) If $D_{y,y^{\prime }}^{k}f\left( \mathbf{x}\right) \leq c$ for all $%
	k\in \left[ n\right] $, $y,y^{\prime }\in \mathcal{X}$ and $\mathbf{x}\in 
	\mathcal{X}^{n}$, then $\mathbb{E}\left[ e^{f\left( \mathbf{X}\right) -%
		\mathbb{E}\left[ f\left( \mathbf{X}^{\prime }\right) \right] }\right] \leq
	e^{nc^{2}/8}.$
	
	(iii) If there is $b\in \left( 0,2\right) $, such that for all $k\in \left[ n%
	\right] $ and $\mathbf{x}\in \mathcal{X}^{n}$\textbf{\ }we have $f\left( 
	\mathbf{x}\right) -\mathbb{E}_{X^{\prime }\sim \mu }\left[ f\left(
	S_{X^{\prime }}^{k}\mathbf{x}\right) \right] \leq b$, then with $%
	v_{k}=\sup_{x\in \mathcal{X}^{n}}\mathbb{E}_{X\sim \mu }\left[ \left(
	f\left( S_{X}^{k}\mathbf{x}\right) -\mathbb{E}_{X^{\prime }\sim \mu }\left[
	f\left( S_{X^{\prime }}^{k}\mathbf{x}\right) \right] \right) ^{2}\right] $%
	\begin{equation*}
		\mathbb{E}_{\mathbf{X}}\left[ e^{f\left( \mathbf{X}\right) -\mathbb{E}\left[
			f\left( \mathbf{X}^{\prime }\right) \right] }\right] \leq \exp \left( \frac{1%
		}{2-b}\sum_{k=1}^{n}v_{k}\right) .
	\end{equation*}
\end{proposition}

Notice that the second conclusion is instrumental in the usual proof of
McDiarmid's (or bounded-difference-) inequality (\cite{McDiarmid98}).

\subsection{Bounded differences\label{Section Bounded Differences}}

In the simplest case the Hamiltonian satisfies a bounded difference
condition as in (ii) above. Then the only minor complication is to show that
the logarithm of the partition function inherits this property. This is
dealt with in the following lemma.

\begin{lemma}
	\label{Lemma bounded differences}Suppose $H:\mathcal{H}\times \mathcal{X}%
	^{n}\rightarrow \mathbb{R}$ and that for all $h\in \mathcal{H}$, $k\in \left[
	n\right] $, $y$, $y^{\prime }$ $\in \mathcal{X}$ and $\mathbf{x}\in \mathcal{%
		X}^{n}$ we have $D_{y,y^{\prime }}^{k}H\left( h,\mathbf{x}\right) \leq c$.
	Let%
	\begin{equation*}
		H_{Q}\left( h,\mathbf{x}\right) =H\left( h,\mathbf{x}\right) -\ln Z\left( 
		\mathbf{x}\right) \text{ with }Z\left( \mathbf{x}\right) =\int_{\mathcal{H}%
		}e^{H\left( h,\mathbf{x}\right) }d\pi \left( h\right) .
	\end{equation*}%
	Then $\forall h\in \mathcal{H},k\in \left[ n\right] $, $y$, $y^{\prime }$ $%
	\in \mathcal{X}$, $\mathbf{x}\in \mathcal{X}^{n}$ we have $D_{y,y^{\prime
	}}^{k}H_{Q}\left( h,\mathbf{x}\right) \leq 2c.$
\end{lemma}

\begin{proof}
	This follows from the linearity of the partial difference operator $%
	D_{y,y^{\prime }}^{k}$ and%
	\begin{eqnarray*}
		D_{y^{\prime }y}^{k}\ln Z\left( \mathbf{x}\right)  &=&\ln \frac{Z\left(
			S_{y^{\prime }}^{k}\mathbf{x}\right) }{Z\left( S_{y}^{k}\mathbf{x}\right) }%
		=\ln \frac{\int_{\mathcal{H}}\exp \left( D_{y^{\prime },y}^{k}H\left( h,%
			\mathbf{x}\right) \right) \exp \left( H\left( h,S_{y}^{k}\mathbf{x}\right)
			\right) d\pi \left( h\right) }{\int_{\mathcal{H}}\exp \left( H\left(
			h,S_{y}^{k}\mathbf{x}\right) \right) d\pi \left( h\right) } \\
		&\leq &\ln \sup_{h\in \mathcal{H}}\exp \left( D_{y^{\prime },y}^{k}H\left( h,%
		\mathbf{x}\right) \right) \leq c.
	\end{eqnarray*}
\end{proof}

\begin{theorem}
	\label{Theorem bounded differences}Suppose $H$ is a Hamiltonian for $Q$ and
	that for all $k\in \left[ n\right] $, $h\in \mathcal{H}$, $y$, $y^{\prime }$ 
	$\in \mathcal{X}$ and $\mathbf{x}\in \mathcal{X}^{n}$ we have $%
	D_{y,y^{\prime }}^{k}H\left( h,\mathbf{x}\right) \leq c$ and $h\left(
	y\right) \in \left[ 0,b\right] $. Then
	
	(i) For any $\lambda >0$%
	\begin{equation*}
		\ln \mathbb{E}_{\mathbf{X}\sim \mu ^{n}}\mathbb{E}_{h\sim Q_{\mathbf{X}}}%
		\left[ e^{\lambda \Delta }\right] \leq \sup_{h\in \mathcal{H}}\psi _{\lambda
			\Delta }\left( h\right) \leq \frac{n}{8}\left( \frac{\lambda b}{n}+2c\right)
		^{2}.
	\end{equation*}
	
	(ii) If $\delta >0$ then with probability at least $1-\delta $ as $\mathbf{X}%
	\sim \mu ^{n}$ and $h\sim Q_{\mathbf{X}}$ we have%
	\begin{equation*}
		\Delta \left( h,\mathbf{X}\right) \leq b\left( c+\sqrt{\frac{\ln \left(
				1/\delta \right) }{2n}}\right) .
	\end{equation*}
\end{theorem}

\begin{proof}
	Using the previous lemma $D_{y,y^{\prime }}^{k}\left( \lambda \Delta \left(
	h,\mathbf{x}\right) +H_{Q}\left( h,\mathbf{x}\right) \right) \leq \left(
	\lambda b/n\right) +2c$. Since $\mathbb{E}\left[ \lambda \Delta \left( h,%
	\mathbf{X}\right) \right] =0$, Proposition \ref{Proposition my Martingale}
	(ii) gives, for any $h\in \mathcal{H}$, 
	\begin{equation*}
		\psi _{\lambda \Delta }\left( h\right) =\ln \mathbb{E}\left[ e^{\lambda
			\Delta \left( h,\mathbf{X}\right) +H_{Q}\left( h,\mathbf{X}\right) -\mathbb{E%
			}\left[ H_{Q}\left( h,\mathbf{X}^{\prime }\right) \right] }\right] \leq 
		\frac{n}{8}\left( \frac{\lambda b}{n}+2c\right) ^{2},
	\end{equation*}%
	and the first conclusion follows from Proposition \ref{Proposition Main} (i)
	with $F\left( h,\mathbf{x}\right) =\lambda \Delta \left( h,\mathbf{X}\right) 
	$.
	
	From Proposition \ref{Proposition Main} (ii) we get with probability at
	least $1-\delta $ as $\mathbf{X}\sim \mu ^{n},h\sim Q_{\mathbf{X}}$ that%
	\begin{equation*}
		\Delta \left( h,\mathbf{X}\right) \leq \lambda ^{-1}\left( \frac{n}{8}\left( 
		\frac{\lambda b}{n}+2c\right) ^{2}+\ln \frac{1}{\delta }\right) =\frac{%
			\lambda b^{2}}{8n}+\frac{bc}{2}+\frac{nc^{2}/2+\ln \left( 1/\delta \right) }{%
			\lambda }.
	\end{equation*}%
	Substitution of the optimal value $\lambda =\sqrt{\left( 8n/b^{2}\right)
		\left( nc^{2}/2+\ln \left( 1/\delta \right) \right) }$ and subadditivity of $%
	t\mapsto \sqrt{t}$ give the second conclusion.
\end{proof}

In applications, such as the Gibbs algorithm, we typically have $c=O\left(
1/n\right) $. The previous result is simple and already gives competitive
bounds for several methods, but it does not take into account the properties
of the hypothesis chosen from $Q_{\mathbf{X}}$. To make the method more
flexible we use the Cauchy-Schwarz inequality to write%
\begin{eqnarray}
	\psi _{F}\left( h\right) &=&\ln \mathbb{E}\left[ e^{F\left( h,\mathbf{X}%
		\right) +H_{Q}\left( h,\mathbf{X}\right) -\mathbb{E}\left[ H_{Q}\left( h,%
		\mathbf{X}^{\prime }\right) \right] }\right]  \notag \\
	&\leq &\frac{1}{2}\ln \mathbb{E}\left[ e^{2F\left( h,\mathbf{X}\right) }%
	\right] +\frac{1}{2}\ln \mathbb{E}\left[ e^{2\left( H_{Q}\left( h,\mathbf{X}%
		\right) -\mathbb{E}\left[ H_{Q}\left( h,\mathbf{X}^{\prime }\right) \right]
		\right) }\right]  \label{Cauchy Schwarz trick}
\end{eqnarray}%
and treat the two terms separately. The disadvantage here is an increase of
constants in the bounds. The big advantage is that different types of bounds
can be combined.

The next result is based on this idea. It is similar to Bernstein's
inequality for sums of independent variables.

\begin{theorem}
	\label{Theorem Bernstein} Under the conditions of Theorem \ref{Theorem
		bounded differences} define for each $h\in \mathcal{H}$ its variance $%
	v\left( h\right) :=\mathbb{E}\left[ \left( h\left( X\right) -\mathbb{E}\left[
	h\left( X^{\prime }\right) \right] \right) ^{2}\right] $. Then for $\delta
	>0 $ with probability at least $1-\delta $ in $\mathbf{X}\sim \mu ^{n}$ and $%
	h\sim Q_{\mathbf{X}}$%
	\begin{equation*}
		\Delta \left( h,\mathbf{X}\right) \leq 2\sqrt{v\left( h\right) \left( c^{2}+%
			\frac{\ln \left( 1/\delta \right) }{n}\right) }+b\left( c^{2}+\frac{\ln
			\left( 1/\delta \right) }{n}\right) .
	\end{equation*}
\end{theorem}

Similar to Bernstein's inequality the above gives a better bound if the
chosen hypothesis has a small variance. For the proof we use the following
lemma, which is a direct consequence of Proposition \ref{Proposition my
	Martingale} (iii).

\begin{lemma}
	\label{Lemma Bernstein auxiliary} (Proof in Section \ref{Proof of Lemma
		Bernstein auxiliary}) Assume that for all $h\in \mathcal{H}$ and $x\in 
	\mathcal{X}$, we have $h\left( x\right) \in \left[ 0,b\right] $ and $v\left(
	h\right) $ as in Theorem \ref{Theorem Bernstein} above. Let $\lambda :%
	\mathcal{H}\rightarrow \left( 0,n/b\right) $ and define $F_{\lambda }:%
	\mathcal{H\times X}^{n}\rightarrow \mathbb{R}$ by%
	\begin{equation*}
		F_{\lambda }\left( h,\mathbf{X}\right) =\lambda \left( h\right) \Delta
		\left( h,\mathbf{X}\right) -\frac{\lambda \left( h\right) ^{2}}{1-b\lambda
			\left( h\right) /n}\frac{v\left( h\right) }{n}.
	\end{equation*}%
	Then for all $h\in \mathcal{H}$ we have $\mathbb{E}\left[ e^{2F_{\lambda
		}\left( h,\mathbf{X}\right) }\right] \leq 1$.
\end{lemma}

\begin{proof}[Proof of Theorem \protect\ref{Theorem Bernstein}]
	Let%
	\begin{equation*}
		\lambda \left( h\right) =\frac{\sqrt{nc^{2}+\ln \left( 1/\delta \right) }}{%
			\left( b/n\right) \sqrt{nc^{2}+\ln \left( 1/\delta \right) }+\sqrt{\frac{%
					v\left( h\right) }{n}}}
	\end{equation*}%
	and let $F_{\lambda }$ be as in Lemma \ref{Lemma Bernstein auxiliary}. It
	follows from Lemma \ref{Lemma bounded differences} and Proposition \ref%
	{Proposition my Martingale} (ii) that for all $h\in \mathcal{H}$ we have $%
	\ln \mathbb{E}\left[ e^{2\left( H_{Q}\left( h,\mathbf{X}\right) -\mathbb{E}%
		\left[ H_{Q}\left( h,\mathbf{X}^{\prime }\right) \right] \right) }\right]
	\leq 2nc^{2}$. Then (\ref{Cauchy Schwarz trick}) and Lemma \ref{Lemma
		Bernstein auxiliary} give $\psi _{F}\left( h\right) \leq nc^{2}$. Thus from
	Proposition \ref{Proposition Main} and division by $\lambda \left( h\right) $%
	\begin{equation*}
		\Pr_{X\sim \mu ^{n},h\sim Q_{\mathbf{X}}}\left\{ \Delta \left( h,\mathbf{X}%
		\right) >\frac{\lambda \left( h\right) }{1-b\lambda \left( h\right) /n}\frac{%
			v\left( h\right) }{n}+\frac{nc^{2}+\ln \left( 1/\delta \right) }{\lambda
			\left( h\right) }\right\} <\delta .
	\end{equation*}%
	Inserting the definition of $\lambda \left( h\right) $ and simplifying
	completes the proof.
\end{proof}

At the expense of larger constants the role of the variance in this result
can be replaced by the empirical error, using $v\left( h\right) \leq \mathbb{%
	E}\left[ h\left( X\right) \right] ^{2}\leq b~\mathbb{E}\left[ h\left(
X\right) \right] $ and a simple algebraic inversion, which is given in the
appendix, Section \ref{Section Inversion Lemma}.

\begin{corollary}
	\label{Corollary Bernstein} Under the conditions of Theorem \ref{Theorem
		Bernstein} for $\delta >0$ with probability at least $1-\delta $ in $\mathbf{%
		X}\sim \mu ^{n}$ and $h\sim Q_{\mathbf{X}}$%
	\begin{equation*}
		\Delta \left( h,\mathbf{X}\right) \leq 2\sqrt{\hat{L}\left( h,\mathbf{X}%
			\right) b\left( c^{2}+\frac{\ln \left( 1/\delta \right) }{n}\right) }%
		+5b\left( c^{2}+\frac{\ln \left( 1/\delta \right) }{n}\right) .
	\end{equation*}
\end{corollary}

\subsection{Subgaussian hypotheses\label{Section Subgaussian hypotheses}}

Some of the above extends to unbounded hypotheses. A real random variable $Y$
is $\sigma $-subgaussian for $\sigma >0$, if $\mathbb{E}\left[ \exp \left(
\lambda \left( Y-\mathbb{E}\left[ Y^{\prime }\right] \right) \right) \right]
\leq e^{\lambda ^{2}\sigma ^{2}/2}$ for every $\lambda \in \mathbb{R}$ . The
proof of the following result is given in the appendix (Section \ref{Section
	Proof Subgaussian hypotheses}) and uses ideas very similar to the proofs
above.

\begin{theorem}
	\label{Theorem subgaussian}Let $Q$ have Hamiltonian $H$ and assume that $%
	\forall h\in \mathcal{H}$ there is $\rho \left( h\right) >0$ such that 
	\begin{equation*}
		\forall \lambda \in \mathbb{R}\text{, }\mathbb{E}\left[ e^{\lambda \left(
			h\left( X\right) -\mathbb{E}\left[ h\left( X^{\prime }\right) \right]
			\right) }\right] \leq e^{\frac{\lambda ^{2}\rho \left( h\right) ^{2}}{2}}.
	\end{equation*}%
	Let $\rho =\sup_{h\in \mathcal{H}}\rho \left( h\right) $ and suppose that $%
	\forall \lambda \in \mathbb{R},k\in \left[ n\right] ,h\in \mathcal{H}$%
	\begin{equation*}
		\mathbb{E}\left[ e^{\lambda \left( H\left( h,S_{X}^{k}\mathbf{x}\right) -%
			\mathbb{E}\left[ H\left( h,S_{X^{\prime }}^{k}\mathbf{x}\right) \right]
			\right) }\right] \leq e^{\frac{\lambda ^{2}\sigma ^{2}}{2}}.
	\end{equation*}
	
	(i) Then for any $h\in \mathcal{H}$, $\lambda >0$%
	\begin{equation*}
		\ln \mathbb{E}_{\mathbf{X}\sim \mu ^{n}}\mathbb{E}_{h\sim Q_{\mathbf{X}}}%
		\left[ e^{\lambda \Delta }\right] \leq \psi _{\lambda \Delta }\left(
		h\right) \leq \frac{\left( \lambda \rho \left( h\right) /\sqrt{n}+2\sqrt{n}%
			\sigma \right) ^{2}}{2},
	\end{equation*}%
	and with probability at least $1-\delta $ we have as $\mathbf{X}\sim \mu
	^{n} $ and $h\sim Q_{\mathbf{X}}$ that%
	\begin{equation*}
		\Delta \left( h,\mathbf{X}\right) \leq \rho \left( 2\sigma +\sqrt{\frac{2\ln
				\left( 1/\delta \right) }{n}}\right) .
	\end{equation*}
	
	(ii) With probability at least $1-\delta $ we have as $\mathbf{X}\sim \mu
	^{n}$ and $h\sim Q_{\mathbf{X}}$ that%
	\begin{equation*}
		\Delta \left( h,\mathbf{X}\right) \leq \rho \left( h\right) \left( \sqrt{32}%
		\sigma +\sqrt{\frac{4\ln \left( 1/\delta \right) }{n}}\right) .
	\end{equation*}
\end{theorem}

The assumptions mean that every hypothesis has its own subgaussian parameter
and that the Hamiltonian is subgaussian in every argument if all other
arguments are fixed. The first conclusion parallels the bound for
Hamiltonians with bounded differences in Theorem \ref{Theorem bounded
	differences}, the second conclusion has larger constants, but scales with
the subgaussian parameter of the hypothesis actually chosen from $Q_{\mathbf{%
		X}}$, which can be considerably smaller, similar to the Bernstein-type
inequality Theorem \ref{Theorem Bernstein}.

\section{Applications\label{Section Applications}}

\subsection{The Gibbs algorithm}

The Gibbs distribution for a sample $\mathbf{x}$ is $dQ_{\beta ,\mathbf{x}%
}\left( h\right) =Z^{-1}\exp \left( -\left( \beta /n\right)
\sum_{i=1}^{n}h\left( x_{i}\right) \right) d\pi \left( h\right) $, so the
Hamiltonian is $H\left( h,\mathbf{x}\right) =-\left( \beta /n\right)
\sum_{i=1}^{n}h\left( x_{i}\right) $. As it is the minimizer of the
PAC-Bayesian bounds (\cite{mcallester1999pac}) generalization bounds for the
Gibbs distribution translate to guarantees for these algorithms, although
the exact $\beta $ for the minimizer is known only implicitly. Let us first
assume bounded hypotheses, for simplicity $h:\mathcal{X}\rightarrow \left[
0,1\right] $. Then we can use Theorems \ref{Theorem bounded differences} and %
\ref{Theorem Bernstein} and Corollary \ref{Corollary Bernstein} with $%
c=\beta /n$. Theorem \ref{Theorem bounded differences} gives with
probability at least $1-\delta $ in $\mathbf{X}\sim \mu ^{n}$ and $h\sim
Q_{\beta ,\mathbf{X}}$ that%
\begin{equation}
	\Delta \left( h,\mathbf{X}\right) \leq \frac{\beta }{n}+\sqrt{\frac{\ln
			\left( 1/\delta \right) }{2n}}.  \label{Simple Gibbs bound}
\end{equation}%
We were not able to find this simple bound in the literature. It improves
over 
\begin{equation*}
	\Delta \left( h,\mathbf{X}\right) \leq \frac{4\beta }{n}+\frac{2+\ln \left(
		\left( 1+\sqrt{e}\right) /\delta \right) }{\sqrt{n}}
\end{equation*}%
obtained in (\cite{rivasplata2020pac}, Sec. 2.1 and Lemma 3) not only in
constants, but, more importantly, in its dependence on the confidence
parameter $\delta $. The principal merit of (\ref{Simple Gibbs bound}),
however, lies in the generality and simplicity of its proof (compare the
proof of Lemma 3 in \cite{rivasplata2020pac}).

Upon the substitution $c=\beta /n$ Theorem \ref{Theorem Bernstein} leads to
a variance dependent bound, for which we know of no comparable result.

From Corollary \ref{Corollary Bernstein} we get for the Gibbs algorithm with
probability at least $1-\delta $ in $\mathbf{X}$ and $h\sim Q_{\beta ,%
	\mathbf{x}}$%
\begin{equation}
	\Delta \left( h,\mathbf{X}\right) \leq 2\sqrt{\hat{L}\left( h,\mathbf{X}%
		\right) \left( \frac{\beta ^{2}}{n^{2}}+\frac{\ln \left( 1/\delta \right) }{n%
		}\right) }+5\left( \frac{\beta ^{2}}{n^{2}}+\frac{\ln \left( 1/\delta
		\right) }{n}\right) ,  \label{Inverted Bernstein Gibbs}
\end{equation}%
For hypotheses with small empirical error this approximates a "fast
convergence rate" of $O\left( 1/n\right) $. Comparable bounds in the
literature involve the so-called "little KL-divergence". For two numbers $%
s,t\in \left[ 0,1\right] $ the relative entropy of two Bernoulli variables,
with means $s$ and $t$ respectively, is $kl\left( s,t\right) =s\ln \left(
s/t\right) +\left( 1-s\right) \ln \left( \left( 1-s\right) /\left(
1-t\right) \right) $. Various authors give bounds on $\mathbb{E}_{h\sim
	Q_{\beta ,\mathbf{x}}}\left[ kl\left( \hat{L}\left( h,\mathbf{X}\right)
,L\left( h\right) \right) \right] $ with high probability in the sample. 
\cite{rivasplata2020pac} give 
\begin{equation*}
	\mathbb{E}_{h\sim Q_{\beta ,\mathbf{x}}}\left[ kl\left( \hat{L}\left( h,%
	\mathbf{X}\right) ,L\left( h\right) \right) \right] \leq \frac{2\beta ^{2}}{%
		n^{2}}+\sqrt{2\ln 3}\frac{\beta }{n^{3/2}}+\frac{1}{n}\ln \left( \frac{4%
		\sqrt{n}}{\delta }\right) ,
\end{equation*}%
and there is a similar bound in \cite{dziugaite2018data} and a slightly
weaker one in \cite{lever2013tighter}. The most useful form of these bounds
is obtained using the following inversion rule (\cite{tolstikhin2013pac},
see also \cite{alquier2021user}): if $kl\left( \hat{L}\left( h,\mathbf{x}%
\right) ,L\left( h\right) \right) \leq B$ then $\Delta \left( h,\mathbf{x}%
\right) \leq \sqrt{2\hat{L}\left( h,\mathbf{x}\right) B}+2B$. If this rule
is applied to the $kl$-bound above, it becomes clear, that it is inferior to
(\ref{Inverted Bernstein Gibbs}), not only because of the logarithmic
dependence on $n$, but also because of artifact terms, which are difficult
to interpret, like the superfluous $\beta /n^{3/2}$.

If every $h\left( X\right) $ is $\rho \left( h\right) $-subgaussian and $%
\rho =\sup_{h}\rho \left( h\right) $, then by linearity of the subgaussian
parameter $H\left( h,\mathbf{X}\right) $ is $\rho \beta /n$-subgaussian in
every argument for every $h$, and Theorem \ref{Theorem subgaussian} gives
with probability at least $1-\delta $ in $\mathbf{X}\sim \mu ^{n}$ and $%
h\sim Q_{\beta ,\mathbf{X}}$%
\begin{equation*}
	\Delta \left( h,\mathbf{X}\right) \leq \rho \left( h\right) \left( \frac{4%
		\sqrt{2}\rho \beta }{n}+\sqrt{\frac{4\ln \left( 1/\delta \right) }{n}}%
	\right) .
\end{equation*}%
Recently \cite{aminian2023information} gave a very interesting bound in
probability for sub-gaussian hypotheses, which however is not quite
comparable to the above, as it bounds the posterior expectation of $\Delta $
and relies on a distribution-dependent prior.

\subsection{Randomization of stable algorithms\label{Section Randomization
		of stable algorithms}}

Suppose that $\mathcal{H}$ is parametrized by $\mathbb{R}^{d}$, with $\pi $
being Lebesgue measure. To simplify notation we identify a hypothesis $h\in 
\mathcal{H}$ with its parametrizing vector, so that $h$ is simultaneously a
vector in $\mathbb{R}^{d}$ and a function $h:x\in \mathcal{X}\mapsto h\left(
x\right) \in \mathbb{R}$. Following \cite{rivasplata2018pac} we define the 
\textit{hypothesis sensitivity coefficient} of a vector valued algorithm $A:%
\mathcal{X}^{n}\rightarrow \mathbb{R}^{d}$ as 
\begin{equation*}
	c_{A}=\max_{k\in \left[ n\right] }\sup_{\mathbf{x}\in \mathcal{X}%
		^{n},y,y^{\prime }\in \mathcal{X}}\left\Vert D_{y,y^{\prime }}^{k}A\left( 
	\mathbf{x}\right) \right\Vert \text{.}
\end{equation*}%
In typical applications $c_{A}=O\left( 1/n\right) $ (compare the
SVM-application in \cite{rivasplata2018pac}, as derived originally from \cite%
{Bousquet02}).

Consider first the algorithm arising from the Hamiltonian%
\begin{equation}
	H\left( h,\mathbf{x}\right) =-G\left( h-A\left( \mathbf{x}\right) \right) ,
	\label{exponential randomization}
\end{equation}%
where $G:\mathbb{R}^{d}\rightarrow \left[ 0,\infty \right) $ is any function
with Lipschitz norm $\left\Vert G\right\Vert _{\text{Lip}}$. One computes $%
A\left( \mathbf{x}\right) $ and samples $h$ from the stochastic kernel
proportional to $\exp \left( -G\left( h-A\left( \mathbf{x}\right) \right)
\right) $. By the triangle inequality $H$ satisfies the bounded difference
conditions of Theorems \ref{Theorem bounded differences} and \ref{Theorem
	Bernstein} with $c=\left\Vert G\right\Vert _{\text{Lip}}c_{A}$. If every $%
h\in \mathcal{H}$ (as a function on $\mathcal{X}$) has range in $\left[ 0,1%
\right] $, then, although this algorithm is of a completely different
nature, we immediately recover the generalization guarantees as for the
Gibbs-algorithm with $\beta /n$ replaced by $\left\Vert G\right\Vert _{\text{%
		Lip}}c_{A}$. An obvious example is $G\left( h\right) =\left\Vert
h\right\Vert /\sigma $ for $\sigma >0$.

Another interesting algorithm arises from the Hamiltonian%
\begin{equation*}
	H\left( h,\mathbf{x}\right) =-\frac{\left\Vert h-A\left( \mathbf{x}\right)
		\right\Vert ^{2}}{2\sigma ^{2}},
\end{equation*}%
for $\sigma >0$, corresponding to gaussian randomization. For stochastic
hypotheses there is an elegant treatment by \cite{rivasplata2018pac} using
the PAC-Bayesian theorem, and resulting in the bound (with probability at
least $1-\delta $ as $\mathbf{X}\sim \mu ^{n}$)%
\begin{equation}
	\mathbb{E}_{h\sim Q_{\mathbf{X}}}\left[ kl\left( \hat{L}\left( h,\mathbf{X}%
	\right) ,L\left( h\right) \right) \right] \leq \frac{\frac{nc_{A}^{2}}{%
			2\sigma ^{2}}\left( 1+\sqrt{\frac{1}{2}\ln \left( \frac{1}{\delta }\right) }%
		\right) ^{2}+\ln \left( \frac{2\sqrt{n}}{\delta }\right) }{n}.
	\label{Rivasplata Gaussian Bound}
\end{equation}%
Since the squared norm is not Lipschitz the previous argument does not work,
but with a slight variation of the method we can prove the following result
(proof in Section \ref{Section Proofs for randomization of stable algorithms}%
).

\begin{theorem}
	\label{Theorem Gaussian}Let $\mathcal{H}=\mathbb{R}^{d}$ with Lebesgue
	measure $\pi $. Suppose $Q$ has Hamiltonian \thinspace $H\left( h,\mathbf{X}%
	\right) =-\left\Vert h-A\left( \mathbf{X}\right) \right\Vert ^{2}/2\sigma
	^{2}$, where $A$ has stability coefficient $c_{A}$. Let $\delta >0$ and
	assume that $12nc_{A}^{2}\leq \sigma ^{2}$ and that every $h\in \mathcal{H}$
	(as a function on $\mathcal{X}$) has range in $\left[ 0,1\right] $. Denote
	the variance of $A$ by $\mathcal{V}\left( A\right) =\mathbb{E}\left[
	\left\Vert A\left( \mathbf{X}\right) -\mathbb{E}\left[ A\left( \mathbf{X}%
	^{\prime }\right) \right] \right\Vert ^{2}\right] $. Then
	
	(i) If $n>8$ then $\ln \mathbb{E}_{\mathbf{X}}\left[ \mathbb{E}_{h\sim Q_{%
			\mathbf{X}}}\left[ e^{\left( n/2\right) kl\left( \hat{L}\left( h,\mathbf{X}%
		\right) ,L\left( h\right) \right) }\right] \right] \leq \frac{3}{\sigma ^{2}}%
	\mathcal{V}\left( A\right) +\frac{1}{2}\ln \left( 2\sqrt{n}\right) .$
	
	(ii) If $n>8$ then with probability at least $1-\delta $ as $\mathbf{X}\sim
	\mu ^{n}$ 
	\begin{equation*}
		\mathbb{E}_{h\sim Q_{\mathbf{X}}}\left[ kl\left( \hat{L}\left( h,\mathbf{X}%
		\right) ,L\left( h\right) \right) \right] \leq \frac{\frac{6}{\sigma ^{2}}%
			\mathcal{V}\left( A\right) +\ln \left( 2\sqrt{n}\right) +2\ln \left(
			1/\delta \right) }{n}.
	\end{equation*}
	
	(iii) With probability at least $1-\delta $ as $\mathbf{X}\sim \mu ^{n}$ and 
	$h\sim Q_{\mathbf{X}}$%
	\begin{equation*}
		\Delta \left( h,\mathbf{X}\right) \leq \sqrt{\frac{\left( 3/\sigma
				^{2}\right) \mathcal{V}\left( A\right) +\ln \left( 1/\delta \right) }{n}}.
	\end{equation*}
	
	(iv) Let $v\left( h\right) $ be the variance of $h$, defined as in Theorem %
	\ref{Theorem Bernstein}. Then with probability at least $1-\delta $ as $%
	\mathbf{X}\sim \mu ^{n}$ and $h\sim Q_{\mathbf{X}}$%
	\begin{equation*}
		\Delta \left( h,\mathbf{X}\right) \leq 2\sqrt{v\left( h\right) \frac{\left(
				3/\sigma ^{2}\right) \mathcal{V}\left( A\right) +\ln \left( 1/\delta \right) 
			}{n}}+\frac{\left( 3/\sigma ^{2}\right) \mathcal{V}\left( A\right) +\ln
			\left( 1/\delta \right) }{n}.
	\end{equation*}
\end{theorem}

The expected $kl$-bound (ii) is given only for direct comparison with (\ref%
{Rivasplata Gaussian Bound}). (iii) and (vi) are stronger, not only by being
disintegrated, but also by avoiding the logarithmic dependence on $n$.

In comparison to (\ref{Rivasplata Gaussian Bound}) (ii) has slightly larger
constants and we require that $12nc_{A}^{2}\leq \sigma ^{2}$. The latter
assumption is mild and holds for sufficiently large $n$ if $%
nc_{A}^{2}\rightarrow 0$ as $n\rightarrow \infty $ (in applications of (\ref%
{Rivasplata Gaussian Bound}) $c_{A}=O\left( 1/n\right) $), but $nc_{A}^{2}$
may even remain bounded away from zero for $12nc_{A}^{2}\leq \sigma ^{2}$ to
hold. On the other hand $\mathcal{V}\left( A\right) =\mathbb{E}\left[
\left\Vert A\left( \mathbf{X}\right) -\mathbb{E}\left[ A\left( \mathbf{X}%
^{\prime }\right) \right] \right\Vert ^{2}\right] $ is always bounded above
by $nc_{A}^{2}$ (see the proof of Lemma 6 in \cite{rivasplata2018pac}), so
we recover (\ref{Rivasplata Gaussian Bound}) from (ii), while our bound can
take advantage of benign distributions. In fortunate cases $\mathcal{V}%
\left( A\right) $ can be arbitrarily close to zero, while the $nc_{A}^{2}$
in (\ref{Rivasplata Gaussian Bound}) is a consequence of the use of
McDiarmid's inequality in the proof, and (\ref{Rivasplata Gaussian Bound})
remains a worst case bound.

The bound (iv) can be inverted as in Corollary \ref{Corollary Bernstein} to
give faster rates for small empirical errors, but without the logarithmic
dependence in $n$ as in the inverted version of (ii).

\subsection{PAC-Bayes bounds with data-dependent priors}

We quote Theorem 1 (ii) in (\cite{rivasplata2020pac}). For the convenience
of the reader we give a proof in the appendix (Section \ref{Section Proofs
	for PAC-Bayes bounds}).

\begin{theorem}
	\label{Theorem Rivasplata}Let $F:\mathcal{H}\times \mathcal{X}%
	^{n}\rightarrow R$ be measurable. With probability at least $1-\delta $ in
	the draw of $\mathbf{X}\sim \mu ^{n}$ we have for all $P\in \mathcal{P}%
	\left( \mathcal{H}\right) $%
	\begin{equation*}
		\mathbb{E}_{h\sim P}\left[ F\left( h,\mathbf{X}\right) \right] \leq KL\left(
		P,Q_{\mathbf{X}}\right) +\ln \mathbb{E}_{\mathbf{X}}\left[ \mathbb{E}_{h\sim
			Q_{\mathbf{X}}}\left[ e^{F\left( h,\mathbf{X}\right) }\right] \right] +\ln
		\left( 1/\delta \right) .
	\end{equation*}
\end{theorem}

By substitution of our bounds on $\ln \mathbb{E}_{\mathbf{X}}\left[ \mathbb{E%
}_{h\sim Q_{\mathbf{X}}}\left[ e^{F\left( h,\mathbf{X}\right) }\right] %
\right] $ we obtain raw forms of PAC-Bayesian bounds with prior $Q_{\mathbf{X%
}}$ for all the Hamiltonian algorithms considered above. But since the final
form often involves optimizations, some care is needed. In the simplest case
let $F\left( h,\mathbf{X}\right) =\left( n/2\right) kl\left( \hat{L}\left( h,%
\mathbf{X}\right) ,L\left( h\right) \right) $, substitute (i) of Theorem \ref%
{Theorem Gaussian} above and divide by $n/2$, to prove the following.

\begin{theorem}
	\label{Theorem Gaussian PAC-Bayes} Under the conditions of Theorem \ref%
	{Theorem Gaussian} we have with probability at least $1-\delta $ in $\mathbf{%
		X}\sim \mu ^{n}$ for all $P\in \mathcal{P}\left( \mathcal{H}\right) $ that%
	\begin{equation*}
		\mathbb{E}_{h\sim P}\left[ kl\left( \hat{L}\left( h,\mathbf{X}\right)
		,L\left( h\right) \right) \right] \leq \frac{2KL\left( P,Q_{\mathbf{X}%
			}\right) +\frac{6}{\sigma ^{2}}\mathcal{V}\left( A\right) +\ln \left( 2\sqrt{%
				n}\right) +2\ln \left( 1/\delta \right) }{n}.
	\end{equation*}
\end{theorem}

It applies to the case, when the prior is an isotropic Gaussian, centered on
the output of the algorithm $A$, a method related to the methods in \cite%
{dziugaite2018data} and \cite{perez2021tighter}. Section \ref{Section Proofs
	for PAC-Bayes bounds} sketches how PAC-Bayesian bounds analogous to \ref%
{Theorem Gaussian} (iii)\ and (iv) are obtained.

\section{Conclusion and future directions}

The paper presented a method to bound the generalization gap for randomly
generated and deterministically executed hypotheses.

One can probably prove an analogue to Theorem \ref{Theorem bounded
	differences} for non-iid data generated by a uniformly ergodic Markov chain,
by using Marton's coupling method as for example in \cite%
{paulin2015concentration}.

It also appears possible to apply the method to iterated stochastic
algorithms, where the randomization of a stable "microalgorithm" is
repeated, as with stochastic gradient Langevin descent (SGLD). Under
appropriate conditions this might give bounds of the generalization gap
along the entire optimization path.

An obvious challenge is to give bounds for the Gibbs algorithm in the limit $%
\beta \rightarrow \infty $, or, more generally, in the regime $\beta >n$. It
is unlikely that the methods of this paper can be successfully applied to
this problem without very strong and unnatural assumptions.\newpage 


\newpage

\bibliographystyle{abbrvnat}

\newpage

\appendix

\section{Remaining proofs of Section \protect\ref{Section Main results}}

\subsection{Markov's inequality\label{Section Markovs inequality}}

We use the following consequence of Markov's inequality.

\begin{lemma}
	\label{Markov's inequality} For any real random variable $Y$ and $\delta >0$
	we have%
	\begin{equation*}
		\Pr \left\{ Y>\ln \mathbb{E}\left[ e^{Y}\right] +\ln \left( 1/\delta \right)
		\right\} \leq \delta .
	\end{equation*}
\end{lemma}

\begin{proof}
	From Markov's inequality $\Pr \left\{ e^{Y}>\mathbb{E}\left[ e^{Y}\right]
	/\delta \right\} \leq \delta $. Take logarithms.
\end{proof}

\subsection{Proof of Proposition \protect\ref{Proposition my Martingale} 
	\label{Section Proof of proposition Martingale}}

\begin{lemma}
	\label{Lemma Steiger} (i) Let $\varphi \left( t\right) =\left(
	e^{t}-t-1\right) /t^{2}$ if $t\neq 0$. Then the function $\varphi $ is
	increasing, and if the random variable $X$ satisfies $\mathbb{E}\left[ X%
	\right] =0$ and $X\leq b$ for $b>0$, then%
	\begin{equation*}
		\mathbb{E}\left[ e^{X}\right] \leq e^{\varphi \left( b\right) \mathbb{E}%
			\left[ X^{2}\right] }.
	\end{equation*}
	
	(ii) $\varphi \left( t\right) \leq 1/\left( 2-t\right) $ for $0\leq t<2$.
\end{lemma}

\begin{proof}
	Part (i) is Lemma 2.8 in \cite{McDiarmid98}. (ii) follows from a term by
	term comparison of the power series 
	\begin{equation*}
		\varphi \left( t\right) =\sum_{k=0}^{\infty }\frac{t^{k}}{\left( k+2\right) !%
		}\text{ and }\frac{1}{2-t}=\sum_{k=0}^{\infty }\frac{t^{k}}{2^{k+1}}.
	\end{equation*}
\end{proof}

\begin{proposition}[Restatement of Proposition \protect\ref{Proposition my
		Martingale}]
	Let $X,X_{1},...,X_{n}$ be iid random variables with values in $\mathcal{X}$%
	, $\mathbf{X}=\left( X_{1},...,X_{n}\right) $ and $f:\mathcal{X}%
	^{n}\rightarrow \mathbb{R}$ measurable.
	
	(i) If $f$ is such that for all $k\in \left[ n\right] $, $\mathbf{x}\in 
	\mathcal{X}^{n}$ we have $\mathbb{E}_{X}\left[ e^{f\left( S_{X}^{k}\mathbf{x}%
		\right) -\mathbb{E}_{X^{\prime }}\left[ f\left( S_{X^{\prime }}^{k}\mathbf{x}%
		\right) \right] }\right] \leq e^{r^{2}}$, then $\mathbb{E}\left[ e^{f\left( 
		\mathbf{X}\right) -\mathbb{E}\left[ f\left( \mathbf{X}^{\prime }\right) %
		\right] }\right] \leq e^{nr^{2}}.$
	
	(ii) If $D_{y,y^{\prime }}^{k}f\left( \mathbf{x}\right) \leq c$ for all $%
	k\in \left[ n\right] $, $y,y^{\prime }\in \mathcal{X}$ and $\mathbf{x}\in 
	\mathcal{X}^{n}$, then $\mathbb{E}\left[ e^{f\left( \mathbf{X}\right) -%
		\mathbb{E}\left[ f\left( \mathbf{X}^{\prime }\right) \right] }\right] \leq
	e^{nc^{2}/8}.$
	
	(iii) If there is $b\in \left( 0,2\right) $, such that for all $k\in \left[ n%
	\right] $ and $\mathbf{x}\in \mathcal{X}^{n}$\textbf{\ }we have $f\left( 
	\mathbf{x}\right) -\mathbb{E}_{X^{\prime }\sim \mu }\left[ f\left(
	S_{X^{\prime }}^{k}\mathbf{x}\right) \right] \leq b$, then with $%
	v_{k}=\sup_{x\in \mathcal{X}^{n}}\mathbb{E}_{X\sim \mu }\left[ \left(
	f\left( S_{X}^{k}\mathbf{x}\right) -\mathbb{E}_{X^{\prime }\sim \mu }\left[
	f\left( S_{X^{\prime }}^{k}\mathbf{x}\right) \right] \right) ^{2}\right] $%
	\begin{equation*}
		\mathbb{E}_{\mathbf{X}}\left[ e^{f\left( \mathbf{X}\right) -\mathbb{E}\left[
			f\left( \mathbf{X}^{\prime }\right) \right] }\right] \leq \exp \left( \frac{1%
		}{2-b}\sum_{k=1}^{n}v_{k}\right) .
	\end{equation*}
\end{proposition}

\begin{proof}
	(i) For $S\subseteq \left[ n\right] $ we write $\mathbb{E}_{S}\left[ .\right]
	=\mathbb{E}\left[ .|\left\{ X_{i}\right\} _{i\notin S}\right] $, so $\mathbb{%
		E}_{S}\left[ .\right] $ is integration over all variables in $S$. By
	independence $\left\{ \mathbb{E}_{S}\left[ .\right] :S\subseteq \left[ n%
	\right] \right\} $ is a set of commuting projections and $\mathbb{E}_{S_{1}}%
	\left[ \mathbb{E}_{S_{2}}\left[ .\right] \right] =\mathbb{E}_{S_{1}\cup
		S_{2}}\left[ .\right] $. $\mathbb{E}_{\left[ k\right] }$ is expectation in
	all variables up to $X_{k}$, and $\mathbb{E}_{\left\{ k\right\} }$ is
	expectation only in $X_{k}$. The assumption therefore reads%
	\begin{equation*}
		\mathbb{E}_{\left\{ k\right\} }\left[ e^{f\left( X\right) -\mathbb{E}%
			_{\left\{ k\right\} }\left[ f\left( X\right) \right] }\right] \leq e^{r^{2}}.
	\end{equation*}%
	Using $\mathbb{E}_{\left[ k-1\right] }\left[ \mathbb{E}_{\left[ k\right] }%
	\left[ f\left( X\right) \right] \right] =\mathbb{E}_{\left[ k-1\right] }%
	\left[ \mathbb{E}_{\left\{ k\right\} }\left[ f\left( X\right) \right] \right]
	$.We have the telescopic expansion%
	\begin{eqnarray*}
		f\left( X\right) -\mathbb{E}\left[ f\left( X^{\prime }\right) \right] 
		&=&\sum_{k=1}^{n}\mathbb{E}_{\left[ k-1\right] }\left[ f\left( X\right) %
		\right] -\mathbb{E}_{\left[ k\right] }\left[ f\left( X\right) \right]  \\
		&=&\sum_{k=1}^{n}\mathbb{E}_{\left[ k-1\right] }\left[ f\left( X\right) -%
		\mathbb{E}_{\left\{ k\right\} }\left[ f\left( X\right) \right] \right] ,
	\end{eqnarray*}%
	We claim that for all $m$, $0\leq m\leq n$%
	\begin{equation*}
		\mathbb{E}\left[ e^{f\left( X\right) -\mathbb{E}\left[ f\left( X^{\prime
			}\right) \right] }\right] \leq e^{mr^{2}}\mathbb{E}\left[ \exp \left(
		\sum_{k=m+1}^{n}\mathbb{E}_{\left[ k-1\right] }\left[ f\left( X\right) -%
		\mathbb{E}_{\left\{ k\right\} }\left[ f\left( X\right) \right] \right]
		\right) \right] 
	\end{equation*}%
	from which the proposition follows with $m=n$. Because of above telescopic
	expansion the claim is true for $m=0$, and we assume it to hold for $m-1$.
	Then%
	\begin{align*}
		& \mathbb{E}\left[ e^{f\left( X\right) -\mathbb{E}\left[ f\left( X^{\prime
			}\right) \right] }\right]  \\
		& \leq e^{\left( m-1\right) r^{2}}\mathbb{E}\left[ \exp \left( \sum_{k=m}^{n}%
		\mathbb{E}_{\left[ k-1\right] }\left[ f\left( X\right) -\mathbb{E}_{\left\{
			k\right\} }\left[ f\left( X\right) \right] \right] \right) \right]  \\
		& =e^{\left( m-1\right) r^{2}}\mathbb{E}\left[ \exp \left( \mathbb{E}_{\left[
			m-1\right] }\left[ f\left( X\right) -\mathbb{E}_{\left\{ m\right\} }\left[
		f\left( X\right) \right] +\sum_{k=m+1}^{n}\mathbb{E}_{\left[ k-1\right] }%
		\left[ f\left( X\right) -\mathbb{E}_{\left\{ k\right\} }\left[ f\left(
		X\right) \right] \right] \right] \right) \right] 
	\end{align*}%
	because the later terms depend only on the variables $X_{m+1},...,X_{n}$. By
	Jensen's inequality the last expression is bounded by%
	\begin{align*}
		& e^{\left( m-1\right) r^{2}}\mathbb{E}\left[ \exp \left( f\left( X\right) -%
		\mathbb{E}_{\left\{ m\right\} }\left[ f\left( X\right) \right]
		+\sum_{k=m+1}^{n}\mathbb{E}_{\left[ k-1\right] }\left[ f\left( X\right) -%
		\mathbb{E}_{\left\{ k\right\} }\left[ f\left( X\right) \right] \right]
		\right) \right]  \\
		& =e^{\left( m-1\right) r^{2}}\mathbb{E}\left[ e^{f\left( X\right) -\mathbb{E%
			}_{\left\{ m\right\} }\left[ f\left( X\right) \right] }\exp \left(
		\sum_{k=m+1}^{n}\mathbb{E}_{\left[ k-1\right] }\left[ f\left( X\right) -%
		\mathbb{E}_{\left\{ k\right\} }\left[ f\left( X\right) \right] \right]
		\right) \right]  \\
		& =e^{\left( m-1\right) r^{2}}\mathbb{E}\left[ \mathbb{E}_{\left\{ m\right\}
		}\left[ e^{f\left( X\right) -\mathbb{E}_{\left\{ m\right\} }\left[ f\left(
			X\right) \right] }\right] \exp \left( \sum_{k=m+1}^{n}\mathbb{E}_{\left[ k-1%
			\right] }\left[ f\left( X\right) -\mathbb{E}_{\left\{ k\right\} }\left[
		f\left( X\right) \right] \right] \right) \right] ,
	\end{align*}%
	again because the later terms do not depend on $X_{m}$, and by assumption
	the last expression is bounded by%
	\begin{equation*}
		e^{mr^{2}}\mathbb{E}\left[ \exp \left( \sum_{k=m+1}^{n}\mathbb{E}_{\left[ k-1%
			\right] }\left[ f\left( X\right) -\mathbb{E}_{\left\{ k\right\} }\left[
		f\left( X\right) \right] \right] \right) \right] ,
	\end{equation*}%
	which completes the induction and the proof of (i).
	
	(ii) follows from (i) and Hoeffding's lemma (Lemma 2.2 in \cite{Boucheron13}%
	) which says that 
	\begin{equation*}
		\mathbb{E}_{X\sim \mu }\left[ e^{f\left( S_{X}^{k}\mathbf{x}\right) -\mathbb{%
				E}_{X^{\prime }\sim \mu }\left[ f\left( S_{X^{\prime }}^{k}\mathbf{x}\right) %
			\right] }\right] \leq e^{\frac{c^{2}}{8}},
	\end{equation*}%
	if $f\left( S_{y}^{k}\mathbf{x}\right) $ as a function of $y$ has range in a
	set of diameter $c$.
	
	(iii) Follows from (i) and Lemma \ref{Lemma Steiger} since%
	\begin{equation*}
		\mathbb{E}_{X\sim \mu }\left[ e^{f\left( S_{X}^{k}\mathbf{x}\right) -\mathbb{%
				E}_{X^{\prime }\sim \mu }\left[ f\left( S_{X^{\prime }}^{k}\mathbf{x}\right) %
			\right] }\right] \leq e^{\varphi \left( b\right) v_{k}}.
	\end{equation*}
\end{proof}

\subsection{Proof of Lemma \protect\ref{Lemma Bernstein auxiliary}\label%
	{Proof of Lemma Bernstein auxiliary}}

\begin{lemma}[Restatement of Lemma \protect\ref{Lemma Bernstein auxiliary}]
	Assume that for all $h\in \mathcal{H}$ and $x\in \mathcal{X}$, we have $%
	h\left( x\right) \in \left[ 0,b\right] $. Let $\lambda :\mathcal{H}%
	\rightarrow \left( 0,n/b\right) $ and define $F_{\lambda }:\mathcal{H\times X%
	}^{n}\rightarrow \mathbb{R}$ by%
	\begin{equation*}
		F_{\lambda }\left( h,\mathbf{X}\right) =\lambda \left( h\right) \Delta
		\left( h,\mathbf{X}\right) -\frac{\lambda \left( h\right) ^{2}}{1-b\lambda
			\left( h\right) /n}\frac{v\left( h\right) }{n}.
	\end{equation*}%
	Then for every $h\in \mathcal{H}$ we have $\mathbb{E}\left[ e^{2F_{\lambda
		}\left( h,\mathbf{X}\right) }\right] \leq 1$.
\end{lemma}

\begin{proof}
	For every $h\in \mathcal{H}$ we have $2b\lambda \left( h\right) /n\in \left(
	0,2\right) $. Also $\forall \mathbf{x}\in \mathcal{X}^{n}$, and $\forall
	h\in \mathcal{H}$%
	\begin{equation*}
		\Delta \left( h,\mathbf{x}\right) -\mathbb{E}_{X^{\prime }\sim \mu }\left[
		S_{X^{\prime }}^{k}\Delta \left( h,\mathbf{x}\right) \right] =\frac{1}{n}%
		\left( \mathbb{E}\left[ h\left( X^{\prime }\right) \right] -h\left(
		x_{k}\right) \right) \leq \frac{b}{n}.
	\end{equation*}%
	Thus for every $h\in \mathcal{H}$ we can apply Proposition \ref{Proposition
		my Martingale} (iii) to $f\left( \mathbf{x}\right) =2\lambda \left( h\right)
	\Delta \left( h,\mathbf{x}\right) $ and obtain%
	\begin{eqnarray*}
		\mathbb{E}\left[ e^{2F\left( h,\mathbf{X}\right) }\right] ^{1/2} &=&\mathbb{E%
		}\left[ \exp \left( 2\lambda \left( h\right) \Delta \left( h,\mathbf{X}%
		\right) \right) \right] ^{1/2}\exp \left( \frac{-\lambda \left( h\right) ^{2}%
		}{1-b\lambda \left( h\right) /n}\frac{v\left( h\right) }{n}\right) \\
		&\leq &\exp \left( \frac{2\lambda \left( h\right) ^{2}}{2-2b\lambda \left(
			h\right) /n}\frac{v\left( h\right) }{n}\right) \exp \left( \frac{-\lambda
			\left( h\right) ^{2}}{1-b\lambda \left( h\right) /n}\frac{v\left( h\right) }{%
			n}\right) =1.
	\end{eqnarray*}
\end{proof}

\subsection{Proof of Corollary \protect\ref{Corollary Bernstein}\label%
	{Section Inversion Lemma}}

\begin{lemma}
	\label{Lemma Inversion of Bernstein} Let $L,\hat{L},A\geq 0$ and assume that%
	\begin{equation*}
		L\leq \hat{L}+2\sqrt{L}\sqrt{A}+A
	\end{equation*}%
	Then $L\leq \hat{L}+2\sqrt{\hat{L}A}+5A$
\end{lemma}

\begin{proof}
	\begin{eqnarray*}
		L &\leq &\hat{L}+2\sqrt{L}\sqrt{A}+A\iff L-2\sqrt{L}\sqrt{A}+A\leq \hat{L}+2A
		\\
		&\iff &\left( \sqrt{L}-\sqrt{A}\right) ^{2}\leq \hat{L}+2A\implies \sqrt{L}%
		\leq \sqrt{\hat{L}+2A}+\sqrt{A} \\
		&\iff &L\leq \left( \sqrt{\hat{L}+2A}+\sqrt{A}\right) ^{2}\leq \hat{L}+2%
		\sqrt{\hat{L}A}+\left( 3+\sqrt{2}\right) A.
	\end{eqnarray*}%
	The lemma follows from $3+\sqrt{2}\leq 5$.
\end{proof}

To get Corollary \ref{Corollary Bernstein} apply this to Theorem \ref%
{Theorem Bernstein}.

\subsection{Proof of Theorem \protect\ref{Theorem subgaussian}\label{Section
		Proof Subgaussian hypotheses}}

\begin{lemma}
	\label{Lemma subgaussian} (From \cite{buldygin1980sub}) (i) if $Y$ is $%
	\sigma $-subgaussian then $\mathbb{E}\left[ \left( Y-\mathbb{E}\left[ Y%
	\right] \right) ^{2}\right] \leq \sigma ^{2}$. (ii) if $Y_{1}$ and $Y_{2}$
	are $\sigma _{1}\,$- and $\sigma _{2}$-subgaussian respectively, the $%
	Y_{1}+Y_{2}$ is $\sigma _{1}+\sigma _{2}$-subgaussian.
\end{lemma}

The next lemma shows that the log-partition function $\ln Z\left( \mathbf{X}%
\right) $ is exponentially concentrated, whenever the Hamiltonian $H\left( h,%
\mathbf{X}\right) $ is subgaussian uniformly in $h$.

\begin{lemma}
	Let $p\geq 1$ and $H\left( h,\mathbf{X}\right) $ be $\sigma $-subgaussian
	for every $h\in \mathcal{H}$ and%
	\begin{equation*}
		Z\left( \mathbf{x}\right) =\int_{\mathcal{H}}e^{H\left( h,\mathbf{x}\right)
		}d\pi \left( h\right) .
	\end{equation*}
	
	(i) Then $\ln \mathbb{E}\left[ e^{p\left( -\ln Z\left( \mathbf{X}\right) +%
		\mathbb{E}\left[ \ln Z\left( \mathbf{X}^{\prime }\right) \right] \right) }%
	\right] \leq p^{2}\sigma ^{2}$.
	
	(ii) If $f:\mathcal{X}^{n}\rightarrow \mathbb{R}$ is $\rho $-subgaussian
	then 
	\begin{equation*}
		\ln \mathbb{E}\left[ e^{p\left( f\left( \mathbf{X}\right) -\mathbb{E}\left[
			f\left( \mathbf{X}^{\prime }\right) \right] -\ln Z\left( \mathbf{X}\right) +%
			\mathbb{E}\left[ \ln Z\left( \mathbf{X}^{\prime }\right) \right] \right) }%
		\right] \leq p^{2}\left( \rho +\sigma \right) ^{2}.
	\end{equation*}
\end{lemma}

Since the inequalities are given only for $p\geq 1$ they do not quite imply
that $\ln Z$ itself is subgaussian.\bigskip

\begin{proof}
	We only need to prove (ii), which implies (i) by setting $f\equiv 0$. By
	Jensen's inequality%
	\begin{eqnarray*}
		\mathbb{E}_{\mathbf{X}}\left[ e^{p\left( f\left( \mathbf{X}\right) -\mathbb{E%
			}\left[ f\left( \mathbf{X}^{\prime }\right) \right] -\ln Z\left( \mathbf{X}%
			\right) +\mathbb{E}\left[ \ln Z\left( \mathbf{X}^{\prime }\right) \right]
			\right) }\right] &\leq &\mathbb{E}_{\mathbf{XX}^{\prime }}\left[ e^{p\left(
			f\left( \mathbf{X}\right) -f\left( \mathbf{X}^{\prime }\right) -\ln Z\left( 
			\mathbf{X}\right) +\ln Z\left( \mathbf{X}^{\prime }\right) \right) }\right]
		\\
		&=&\mathbb{E}_{\mathbf{X}}\left[ e^{pf\left( \mathbf{X}\right) }Z\left( 
		\mathbf{X}\right) ^{-p}\right] \mathbb{E}_{\mathbf{X}}\left[ e^{-pf\left( 
			\mathbf{X}\right) }Z\left( \mathbf{X}\right) ^{p}\right] .
	\end{eqnarray*}%
	Define a probability measure $\nu $ on $\mathcal{H}$ by $\nu \left( A\right)
	=Z_{\nu }^{-1}\int_{A}e^{\mathbb{E}\left[ H\left( h,\mathbf{X}\right) \right]
	}d\pi \left( h\right) $ for $A\subseteq \mathcal{H}$ measurable$.$with $%
	Z_{\nu }=\int_{\mathcal{H}}e^{\mathbb{E}\left[ H\left( h,\mathbf{X}\right) %
		\right] }d\pi \left( h\right) $. Then 
	\begin{equation*}
		Z\left( \mathbf{X}\right) ^{-p}=\left( \mathbb{E}_{h\sim \nu }\left[
		e^{H\left( h,\mathbf{X}\right) -\mathbb{E}\left[ H\left( h,\mathbf{X}%
			^{\prime }\right) \right] }\right] \right) ^{-p}Z_{\nu }^{p}\leq \mathbb{E}%
		_{h\sim \nu }\left[ e^{p\left( \mathbb{E}\left[ H\left( h,\mathbf{X}^{\prime
			}\right) \right] -H\left( h,\mathbf{X}\right) \right) }\right] Z_{\nu }^{p},
	\end{equation*}%
	by Jensen's inequality, since $t\mapsto t^{-1}$ is convex. Similarly 
	\begin{equation*}
		Z\left( \mathbf{X}\right) ^{p}\leq \mathbb{E}_{h\sim \nu }\left[ e^{p\left(
			H\left( h,\mathbf{X}\right) -\mathbb{E}\left[ H\left( h,\mathbf{X}^{\prime
			}\right) \right] \right) }\right] Z_{\nu }^{-p}.
	\end{equation*}%
	Thus the above inequality can be written%
	\begin{eqnarray*}
		\mathbb{E}_{\mathbf{X}}\left[ e^{p\left( f\left( \mathbf{X}\right) -\mathbb{E%
			}\left[ f\left( \mathbf{X}^{\prime }\right) \right] -\ln Z\left( \mathbf{X}%
			\right) +\mathbb{E}\left[ \ln Z\left( \mathbf{X}^{\prime }\right) \right]
			\right) }\right] &\leq &\mathbb{E}_{\mathbf{X}}\left[ \mathbb{E}_{h\sim \nu }%
		\left[ e^{p\left( f\left( \mathbf{X}\right) +\mathbb{E}\left[ H\left( h,%
			\mathbf{X}^{\prime }\right) \right] -H\left( h,\mathbf{X}\right) \right) }%
		\right] \right] \\
		&&\times \mathbb{E}_{\mathbf{X}}\left[ \mathbb{E}_{h\sim \nu }\left[
		e^{p\left( -f\left( \mathbf{X}\right) +H\left( h,\mathbf{X}\right) -\mathbb{E%
			}\left[ H\left( h,\mathbf{X}^{\prime }\right) \right] \right) }\right] %
		\right] .
	\end{eqnarray*}%
	The first factor can be bounded by%
	\begin{equation*}
		\mathbb{E}_{h\sim \nu }\left[ \mathbb{E}_{\mathbf{X}}\left[ e^{p\left(
			f\left( \mathbf{X}\right) +\mathbb{E}\left[ H\left( h,\mathbf{X}^{\prime
			}\right) \right] -H\left( h,\mathbf{X}\right) \right) }\right] \right] \leq
		e^{p\mathbb{E}\left[ f\left( \mathbf{X}^{\prime }\right) \right] }e^{\frac{%
				p^{2}\left( \rho +\sigma \right) ^{2}}{2}}
	\end{equation*}%
	by the subgaussian assumptions for $f$ and $H$ and Lemma \ref{Lemma
		subgaussian} (ii), and similarly the second factor is bounded by%
	\begin{equation*}
		e^{-p\mathbb{E}\left[ f\left( \mathbf{X}^{\prime }\right) \right] }e^{\frac{%
				p^{2}\left( \rho +\sigma \right) ^{2}}{2}}.
	\end{equation*}%
	Putting the two bounds together completes the proof.\bigskip
\end{proof}

\begin{theorem}[Restatement of Theorem \protect\ref{Theorem subgaussian}]
	Let $Q$ have Hamiltonian $H$ and assume that $\forall h\in \mathcal{H}$
	there is $\rho \left( h\right) >0$ such that 
	\begin{equation*}
		\forall \lambda \in \mathbb{R}\text{, }\mathbb{E}\left[ e^{\lambda \left(
			h\left( X\right) -\mathbb{E}\left[ h\left( X^{\prime }\right) \right]
			\right) }\right] \leq e^{\frac{\lambda ^{2}\rho \left( h\right) ^{2}}{2}},
	\end{equation*}%
	Let $\rho =\sup_{h\in \mathcal{H}}\rho \left( h\right) $ and suppose that $%
	\forall \lambda \in \mathbb{R},k\in \left[ n\right] ,h\in \mathcal{H}$%
	\begin{equation*}
		\mathbb{E}\left[ e^{\lambda \left( H\left( h,S_{X}^{k}\mathbf{x}\right) -%
			\mathbb{E}\left[ H\left( h,S_{X^{\prime }}^{k}\mathbf{x}\right) \right]
			\right) }\right] \leq e^{\frac{\lambda ^{2}\sigma ^{2}}{2}}.
	\end{equation*}
	
	(i) Then for any $h\in \mathcal{H}$, $\lambda >0$%
	\begin{equation*}
		\ln \mathbb{E}_{\mathbf{X}\sim \mu ^{n}}\mathbb{E}_{h\sim Q_{\mathbf{X}}}%
		\left[ e^{\lambda \Delta }\right] \leq \psi _{\lambda \Delta }\left(
		h\right) \leq \frac{\left( \lambda \rho \left( h\right) /\sqrt{n}+2\sqrt{n}%
			\sigma \right) ^{2}}{2},
	\end{equation*}%
	and with probability at least $1-\delta $ we have as $\mathbf{X}\sim \mu
	^{n} $ and $h\sim Q_{\mathbf{X}}$ that%
	\begin{equation*}
		\Delta \left( h,\mathbf{X}\right) \leq \rho \left( 2\sigma +\sqrt{\frac{2\ln
				\left( 1/\delta \right) }{n}}\right) .
	\end{equation*}
	
	(ii) With probability at least $1-\delta $ we have as $\mathbf{X}\sim \mu
	^{n}$ and $h\sim Q_{\mathbf{X}}$ that%
	\begin{equation*}
		\Delta \left( h,\mathbf{X}\right) \leq \rho \left( h\right) \left( \sqrt{32}%
		\sigma +\sqrt{\frac{4\ln \left( 1/\delta \right) }{n}}\right) .
	\end{equation*}%
	\bigskip
\end{theorem}

\begin{proof}
	Let $h\in \mathcal{H}$ be any fixed hypothesis. By assumption and
	Proposition \ref{Proposition my Martingale} (i) $H\left( h,\mathbf{X}\right) 
	$ is $\sqrt{n}\sigma $-subgaussian and $\lambda \Delta \left( h,\mathbf{X}%
	\right) $ is $\lambda \rho \left( h\right) /\sqrt{n}$-subgaussian.
	
	Using the previous lemma (ii) with $p=1$ and $f\left( \mathbf{X}\right)
	=\lambda \Delta \left( h,\mathbf{X}\right) +H\left( h,\mathbf{X}\right) -%
	\mathbb{E}\left[ H\left( h,\mathbf{X}^{\prime }\right) \right] $, which is
	centered and $\lambda \rho \left( h\right) /\sqrt{n}+\sqrt{n}\sigma $%
	-subgaussian, we get that 
	\begin{eqnarray*}
		\psi _{\lambda \Delta }\left( h\right) &=&\ln \mathbb{E}_{\mathbf{X}}\left[
		e^{\lambda \Delta \left( h,\mathbf{X}\right) +H_{Q}\left( h,\mathbf{X}%
			\right) -\mathbb{E}\left[ H_{Q}\left( h,\mathbf{X}^{\prime }\right) \right] }%
		\right] \\
		&=&\ln \mathbb{E}\left[ e^{f\left( \mathbf{X}\right) -\ln Z\left( \mathbf{X}%
			\right) +\mathbb{E}\left[ \ln Z\left( \mathbf{X}^{\prime }\right) \right] }%
		\right] \\
		&\leq &\frac{\left( \lambda \rho \left( h\right) /\sqrt{n}+2\sqrt{n}\sigma
			\right) ^{2}}{2}.
	\end{eqnarray*}%
	With $\rho \left( g\right) \leq \rho $ we get from Proposition \ref%
	{Proposition Main} that with probability at least $1-\delta $ 
	\begin{eqnarray*}
		\Delta \left( h,\mathbf{X}\right) &\leq &\frac{\left( \lambda \rho /\sqrt{n}%
			+2\sqrt{n}\sigma \right) ^{2}}{2\lambda }+\frac{\ln \left( 1/\delta \right) 
		}{\lambda } \\
		&=&\frac{\lambda \rho ^{2}}{2n}+\frac{2n\sigma ^{2}+\ln \left( 1/\delta
			\right) }{\lambda }+2\rho \sigma
	\end{eqnarray*}%
	The optimal choice of $\lambda $ and subadditivity of $t\rightarrow \sqrt{t}$
	give%
	\begin{equation*}
		\Delta \left( h,\mathbf{X}\right) \leq \rho \left( 2\sigma +\sqrt{\frac{2\ln
				\left( 1/\delta \right) }{n}}\right) .
	\end{equation*}
	
	(ii) We proceed as in the proof of Theorem \ref{Theorem Bernstein}. For $%
	H_{Q}\left( h,\mathbf{X}\right) =H\left( h,\mathbf{X}\right) -\ln Z\left( 
	\mathbf{X}\right) $ the previous lemma yields with $f\left( \mathbf{X}%
	\right) =H\left( g,\mathbf{X}\right) -\mathbb{E}\left[ H\left( g,\mathbf{X}%
	^{\prime }\right) \right] $ and $p=2$ that%
	\begin{equation*}
		\mathbb{E}_{\mathbf{X}}\left[ e^{2\left( H_{Q}\left( h,\mathbf{X}\right) -%
			\mathbb{E}\left[ H_{Q}\left( h,\mathbf{X}^{\prime }\right) \right] \right) }%
		\right] \leq e^{8n\sigma ^{2}}.
	\end{equation*}%
	Also 
	\begin{equation*}
		\mathbb{E}_{\mathbf{X}}\left[ e^{2\lambda \Delta \left( h,\mathbf{X}\right) }%
		\right] \leq e^{2\lambda ^{2}\rho \left( h\right) ^{2}/n}.
	\end{equation*}%
	Now define 
	\begin{equation*}
		\lambda \left( h\right) =\sqrt{\left( \frac{n}{\rho \left( h\right) ^{2}}%
			\right) \left( 8n\sigma ^{2}+\ln \left( 1/\delta \right) \right) }
	\end{equation*}%
	and $F\left( h,\mathbf{X}\right) =\lambda \left( h\right) \Delta \left( h,%
	\mathbf{X}\right) -\lambda \left( h\right) ^{2}\rho \left( h\right) ^{2}/n$.
	Then with Cauchy-Schwarz%
	\begin{eqnarray*}
		\psi _{F}\left( h\right) &=&\ln \mathbb{E}\left[ e^{F\left( h,\mathbf{X}%
			\right) +H_{Q}\left( h,\mathbf{X}\right) -\mathbb{E}\left[ H_{Q}\left( h,%
			\mathbf{X}^{\prime }\right) \right] }\right] \\
		&\leq &\ln \left( \left( \mathbb{E}_{\mathbf{X}}\left[ e^{2\lambda \Delta
			\left( h,\mathbf{X}\right) }\right] \right) ^{1/2}e^{-\lambda \left(
			h\right) ^{2}\rho \left( h\right) ^{2}/n}\mathbb{E}_{\mathbf{X}}\left[
		e^{2\left( H_{Q}\left( h,\mathbf{X}\right) -\mathbb{E}\left[ H_{Q}\left( h,%
			\mathbf{X}^{\prime }\right) \right] \right) }\right] ^{1/2}\right) \\
		&\leq &8n\sigma ^{2}.
	\end{eqnarray*}%
	Proposition \ref{Proposition Main} then gives%
	\begin{eqnarray*}
		\Delta \left( h,\mathbf{X}\right) &\leq &\frac{\lambda \left( h\right) \rho
			\left( h\right) ^{2}}{n}+\frac{8n\sigma ^{2}+\ln \left( 1/\delta \right) }{%
			\lambda \left( h\right) } \\
		&=&\sqrt{\frac{\rho \left( h\right) ^{2}}{n}\left( 32n\sigma ^{2}+4\ln
			\left( 1/\delta \right) \right) } \\
		&=&\rho \left( h\right) \left( \sqrt{32}\sigma +\sqrt{\frac{4\ln \left(
				1/\delta \right) }{n}}\right) .
	\end{eqnarray*}%
	The conclusion follows from subadditivity of $t\rightarrow \sqrt{t}$.
\end{proof}

\section{Remaining proofs for Section \protect\ref{Section Applications}}

\subsection{Proof of Theorem \protect\ref{Theorem Gaussian} \label{Section
		Proofs for randomization of stable algorithms}}

We need the following Lemma.

\begin{lemma}
	\label{Lemma Gaussian}Let $w,v\in \mathbb{R}^{d}$ and $\lambda \in \left[
	1,\infty \right) $ then 
	\begin{equation*}
		\mathbb{E}_{x\sim \mathcal{N}\left( w,\sigma ^{2}I\right) }\left[ e^{\left( 
			\frac{-\lambda }{2\sigma ^{2}}\right) \left( \left\Vert x-v\right\Vert
			^{2}-\left\Vert x-w\right\Vert ^{2}\right) }\right] =e^{\left( \frac{%
				2\lambda ^{2}-\lambda }{2\sigma ^{2}}\right) \left\Vert v-w\right\Vert ^{2}}.
	\end{equation*}
\end{lemma}

\begin{proof}
	We can absorb $\sqrt{2}\sigma $ in the definition of the norm. Then by
	translation%
	\begin{eqnarray*}
		\mathbb{E}_{x\sim \mathcal{N}\left( w,I\right) }\left[ e^{-\lambda \left(
			\left\Vert x-v\right\Vert ^{2}-\left\Vert x-w\right\Vert ^{2}\right) }\right]
		&=&\mathbb{E}_{x\sim \mathcal{N}\left( 0,I\right) }\left[ e^{-\lambda \left(
			\left\Vert x-\left( v-w\right) \right\Vert ^{2}-\left\Vert x\right\Vert
			^{2}\right) }\right] \\
		&=&e^{-\lambda \left\Vert v-w\right\Vert ^{2}}\mathbb{E}_{x\sim \mathcal{N}%
			\left( 0,I\right) }\left[ e^{2\lambda \left\langle x,v-w\right\rangle }%
		\right] .
	\end{eqnarray*}%
	Rotating $v-w$ to $\left\Vert v-w\right\Vert e_{1}$, where $e_{1}$ is the
	first basis vector, and using independence of the components gives%
	\begin{eqnarray*}
		\mathbb{E}_{x\sim \mathcal{N}\left( 0,I\right) }\left[ e^{2\lambda
			\left\langle x,v-w\right\rangle }\right] &=&\mathbb{E}_{x\sim \mathcal{N}%
			\left( 0,I\right) }\left[ e^{2\lambda \left\Vert v-w\right\Vert \left\langle
			x,e_{1}\right\rangle }\right] =\frac{1}{\sqrt{2\pi }}\int_{-\infty }^{\infty
		}e^{2\lambda \left\Vert v-w\right\Vert t-t^{2}/2}dt \\
		&=&\frac{e^{2\lambda ^{2}\left\Vert v-w\right\Vert ^{2}}}{\sqrt{2\pi }}%
		\int_{-\infty }^{\infty }e^{-\left( \frac{2\lambda \left\Vert v-w\right\Vert
				-t}{\sqrt{2}}\right) ^{2}}dt=e^{2\lambda ^{2}\left\Vert v-w\right\Vert ^{2}}.
	\end{eqnarray*}%
	Combination with the previous identity and taking $\sqrt{2}\sigma $ back out
	of the norm completes the proof.
\end{proof}

Our proof of Theorem \ref{Theorem Gaussian} uses the following exponential
concentration inequality, implicit in the proof of Theorem 13 in (\cite%
{maurer2006concentration}) and in the proof of Theorem 6.19 in (\cite%
{Boucheron13})

\begin{theorem}
	\label{Theorem self-bounded concentration}Let $f:\mathcal{X}^{n}\rightarrow 
	\mathbb{R}$ and define an operator $D^{2}$ by%
	\begin{equation*}
		\left( D^{2}f\right) \left( \mathbf{x}\right) =\sum_{k=1}^{n}\left( f\left( 
		\mathbf{x}\right) -\inf_{y\in \mathcal{X}}f\left( S_{y}^{k}\mathbf{x}\right)
		\right) ^{2}.
	\end{equation*}%
	If for some $a>0$ and all $\mathbf{x}\in \mathcal{X}^{n}$, $D^{2}f\left( 
	\mathbf{x}\right) \leq af\left( \mathbf{x}\right) $, then for $\lambda \in
	\left( 0,2/a\right) $ 
	\begin{equation*}
		\ln \mathbb{E}\left[ e^{\lambda \left( f\left( \mathbf{X}\right) -\mathbb{E}%
			\left[ f\left( \mathbf{X}^{\prime }\right) \right] \right) }\right] \leq 
		\frac{\lambda ^{2}a\mathbb{E}\left[ f\left( \mathbf{X}\right) \right] }{%
			2-a\lambda }\text{ \ \ or equivalently }\ln \mathbb{E}\left[ e^{\lambda
			f\left( \mathbf{X}\right) }\right] \leq \frac{2\lambda \mathbb{E}\left[
			f\left( \mathbf{X}\right) \right] }{2-a\lambda }.
	\end{equation*}
\end{theorem}

\begin{proof}[Proof of Theorem \protect\ref{Theorem Gaussian}]
	All Gaussians with covariance $\sigma ^{2}I$ have the same normalizing
	factors, so the partition function for $H\left( h,\mathbf{x}\right)
	=-\left\Vert h-A\left( \mathbf{x}\right) \right\Vert ^{2}/\left( 2\sigma
	^{2}\right) $ is also the normalizing factor of $\mathcal{N}\left( \mathbb{E}%
	\left[ A\left( \mathbf{X}\right) \right] ,\sigma ^{2}I\right) $, whence,
	using Cauchy-Schwarz,%
	\begin{align*}
		& \left. \ln \mathbb{E}_{\mathbf{X}}\left[ \mathbb{E}_{h\sim Q_{\mathbf{X}}}%
		\left[ e^{F\left( h,\mathbf{X}\right) }\right] \right] =\ln \mathbb{E}_{%
			\mathbf{X}}\left[ \int_{\mathcal{H}}e^{F\left( h,\mathbf{X}\right)
			+H_{Q}\left( h,\mathbf{X}\right) }d\pi \left( h\right) \right] \right. \\
		& =\ln \mathbb{E}_{\mathbf{X}}\left[ \mathbb{E}_{h\sim \mathcal{N}\left( 
			\mathbb{E}\left[ A\left( \mathbf{X}\right) \right] ,\sigma ^{2}I\right) }%
		\left[ e^{F\left( h,\mathbf{X}\right) -\frac{\left\Vert h-A\left( \mathbf{X}%
				\right) \right\Vert }{2\sigma ^{2}}^{2}+\frac{\left\Vert h-\mathbb{E}\left[
				A\left( \mathbf{X}\right) \right] \right\Vert }{2\sigma ^{2}}^{2}}\right] %
		\right] \\
		& \leq \frac{1}{2}\sup_{h\in \mathcal{H}}\ln \mathbb{E}_{\mathbf{X}}\left[
		e^{2F\left( h,\mathbf{X}\right) }\right] +\frac{1}{2}\ln \mathbb{E}_{\mathbf{%
				X}}\left[ \mathbb{E}_{h\sim \mathcal{N}\left( \mathbb{E}\left[ A\left( 
			\mathbf{X}\right) \right] ,\sigma ^{2}I\right) }\left[ e^{-2\left( \frac{%
				\left\Vert h-A\left( \mathbf{X}\right) \right\Vert }{2\sigma ^{2}}^{2}-\frac{%
				\left\Vert h-\mathbb{E}\left[ A\left( \mathbf{X}\right) \right] \right\Vert 
			}{2\sigma ^{2}}^{2}\right) }\right] \right] \\
		& =:C+B.
	\end{align*}%
	The bound on $C$ depends on the respective choice of $F$ and will be treated
	below. Using Lemma \ref{Lemma Gaussian} the second term is equal to%
	\begin{equation*}
		B=\frac{1}{2}\ln \mathbb{E}_{\mathbf{X}}\left[ e^{\frac{3}{\sigma ^{2}}%
			\left\Vert A\left( \mathbf{X}\right) -\mathbb{E}\left[ A\left( \mathbf{X}%
			\right) \right] \right\Vert ^{2}}\right] .
	\end{equation*}%
	To apply Theorem \ref{Theorem self-bounded concentration} to $f\left( 
	\mathbf{x}\right) =\left\Vert A\left( \mathbf{x}\right) -\mathbb{E}\left[
	A\left( \mathbf{X}\right) \right] \right\Vert ^{2}$ we fix $\mathbf{x}\in 
	\mathcal{X}^{n}$ and $k\in \left[ n\right] $, and let $y\in \mathcal{X}$ be
	a minimizer of $\left\Vert A\left( S_{y}^{k}\mathbf{x}\right) -\mathbb{E}%
	\left[ A\left( \mathbf{X}\right) \right] \right\Vert ^{2}$. Then 
	\begin{align*}
		& \left. \left( f\left( \mathbf{x}\right) -\inf_{y\in \mathcal{X}}f\left(
		S_{y}^{k}\mathbf{x}\right) \right) ^{2}=\left( \left\Vert A\left( \mathbf{x}%
		\right) -\mathbb{E}\left[ A\left( \mathbf{X}\right) \right] \right\Vert
		^{2}-\left\Vert A\left( S_{y}^{k}\mathbf{x}\right) -\mathbb{E}\left[ A\left( 
		\mathbf{X}\right) \right] \right\Vert ^{2}\right) ^{2}\right. \\
		& =\left\langle A\left( \mathbf{x}\right) -A\left( S_{y}^{k}\mathbf{x}%
		\right) ,A\left( \mathbf{x}\right) -\mathbb{E}\left[ A\left( \mathbf{X}%
		\right) \right] +A\left( S_{y}^{k}\mathbf{x}\right) -\mathbb{E}\left[
		A\left( \mathbf{X}\right) \right] \right\rangle ^{2} \\
		& \leq \left\Vert A\left( \mathbf{x}\right) -A\left( S_{y}^{k}\mathbf{x}%
		\right) \right\Vert ^{2}\left( \left\Vert A\left( \mathbf{x}\right) -\mathbb{%
			E}\left[ A\left( \mathbf{X}\right) \right] \right\Vert +\left\Vert A\left(
		S_{y}^{k}\mathbf{x}\right) -\mathbb{E}\left[ A\left( \mathbf{X}\right) %
		\right] \right\Vert \right) ^{2}\leq 4c_{A}^{2}f\left( \mathbf{x}\right) .
	\end{align*}%
	Summing over $k$ we get $D^{2}f\left( \mathbf{x}\right) \leq
	4nc_{A}^{2}f\left( \mathbf{x}\right) $. Since $12nc_{A}^{2}/\sigma ^{2}\leq
	1<2$ we can apply the theorem with $a=4nc_{A}^{2}$ and $\lambda =3/\sigma
	^{2}$ to obtain 
	\begin{eqnarray*}
		B &=&\frac{1}{2}\ln \mathbb{E}\left[ e^{\left( 3/\sigma ^{2}\right)
			\left\Vert A\left( \mathbf{X}\right) -\mathbb{E}\left[ A\left( \mathbf{X}%
			^{\prime }\right) \right] \right\Vert ^{2}}\right] \leq \frac{\left(
			3/\sigma ^{2}\right) \mathbb{E}\left[ \left\Vert A\left( \mathbf{X}\right) -%
			\mathbb{E}\left[ A\left( \mathbf{X}^{\prime }\right) \right] \right\Vert ^{2}%
			\right] }{2-12nc_{A}^{2}/\sigma ^{2}} \\
		&\leq &\frac{3}{\sigma ^{2}}\mathbb{E}\left[ \left\Vert A\left( \mathbf{X}%
		\right) -\mathbb{E}\left[ A\left( \mathbf{X}^{\prime }\right) \right]
		\right\Vert ^{2}\right] =\frac{3}{\sigma ^{2}}\mathcal{V}\left( A\right) .
	\end{eqnarray*}
\end{proof}

(i) Let $F\left( h,\mathbf{X}\right) =\left( n/2\right) kl\left( \hat{L}%
\left( h,\mathbf{X}\right) ,L\left( h\right) \right) $. Then using (\cite%
{maurer2004note}) $C=\sup_{h}\ln \mathbb{E}_{\mathbf{X}}\left[ e^{2F\left( h,%
	\mathbf{X}\right) }\right] /2\leq \left( 1/2\right) \ln \left( 2\sqrt{n}%
\right) $, so from the above%
\begin{equation*}
	\ln \mathbb{E}_{\mathbf{X}}\left[ \mathbb{E}_{h\sim Q_{\mathbf{X}}}\left[
	e^{\left( n/2\right) kl\left( \hat{L}\left( h,\mathbf{X}\right) ,L\left(
		h\right) \right) }\right] \right] \leq \frac{3}{\sigma ^{2}}\mathcal{V}%
	\left( A\right) +\frac{1}{2}\ln \left( 2\sqrt{n}\right) ,
\end{equation*}%
which is (i). By Jensen's inequality 
\begin{equation*}
	\ln \mathbb{E}_{\mathbf{X}}\left[ e^{\left( n/2\right) \mathbb{E}_{h\sim Q_{%
				\mathbf{X}}}\left[ kl\left( \hat{L}\left( h,\mathbf{X}\right) ,L\left(
		h\right) \right) \right] }\right] \leq \ln \mathbb{E}_{\mathbf{X}}\left[ 
	\mathbb{E}_{h\sim Q_{\mathbf{X}}}\left[ e^{\left( n/2\right) kl\left( \hat{L}%
		\left( h,\mathbf{X}\right) ,L\left( h\right) \right) }\right] \right] ,
\end{equation*}%
so part (ii) then follows from Markov's inequality and division by $n/2$.

(iii) Let $F\left( h,\mathbf{X}\right) =\lambda \Delta \left( h,\mathbf{X}%
\right) $ for $\lambda >0$. Using Proposition \ref{Proposition my Martingale}
(ii) we get for all $h\in \mathcal{H}$ that $C=\left( 1/2\right) \ln \mathbb{%
	E}_{\mathbf{X}}\left[ e^{2F\left( h,\mathbf{X}\right) }\right] \leq \lambda
^{2}/\left( 4n\right) $, so 
\begin{equation}
	\ln \mathbb{E}_{\mathbf{X}}\left[ \mathbb{E}_{h\sim Q_{\mathbf{X}}}\left[
	e^{\lambda \Delta \left( h,\mathbf{X}\right) }\right] \right] =\frac{\lambda
		^{2}}{4n}+\frac{3}{\sigma ^{2}}\mathcal{V}\left( A\right) .
	\label{Moment bound Gaussian}
\end{equation}%
Markov's inequality gives with probability at least $1-\delta $ as $\mathbf{X%
}\sim \mu ^{n}$ and $h\sim Q_{\mathbf{X}}$ that%
\begin{equation*}
	\Delta \left( h,\mathbf{X}\right) \leq \frac{\lambda }{4n}+\frac{\frac{3}{%
			\sigma ^{2}}\mathcal{V}\left( A\right) +\ln \left( 1/\delta \right) }{%
		\lambda }.
\end{equation*}%
Optimization of $\lambda $ gives (iii).

(iv) We proceed as in the proof of Theorem \ref{Theorem Bernstein}. Let 
\begin{equation*}
	\lambda \left( h\right) =\frac{\sqrt{\frac{3}{\sigma ^{2}}\mathcal{V}\left(
			A\right) +\ln \left( 1/\delta \right) }}{\left( 1/n\right) \sqrt{\frac{3}{%
				\sigma ^{2}}\mathcal{V}\left( A\right) +\ln \left( 1/\delta \right) }+\sqrt{%
			\frac{v\left( h\right) }{n}}}
\end{equation*}%
and set 
\begin{equation*}
	F_{\lambda }\left( h,\mathbf{X}\right) =\lambda \left( h\right) \Delta
	\left( h,\mathbf{X}\right) -\frac{\lambda \left( h\right) ^{2}}{1-\lambda
		\left( h\right) /n}\frac{v\left( h\right) }{n}.
\end{equation*}%
By Lemma \ref{Lemma Bernstein auxiliary} $2C=\ln \mathbb{E}_{\mathbf{X}}%
\left[ e^{2F_{\lambda }\left( h,\mathbf{X}\right) }\right] \leq 0$, so 
\begin{equation*}
	\ln \mathbb{E}_{\mathbf{X}}\left[ \mathbb{E}_{h\sim Q_{\mathbf{X}}}\left[
	e^{F_{\lambda }\left( h,\mathbf{X}\right) }\right] \right] \leq \frac{3}{%
		\sigma ^{2}}\mathcal{V}\left( A\right) ,
\end{equation*}%
and with probability at least $1-\delta $ as as $\mathbf{X}\sim \mu ^{n}$
and $h\sim Q_{\mathbf{X}}$%
\begin{equation*}
	\Pr \left\{ \Delta \left( h,\mathbf{X}\right) >\frac{\lambda \left( h\right) 
	}{1-\lambda \left( h\right) /n}\frac{v\left( h\right) }{n}+\frac{\frac{3}{%
			\sigma ^{2}}\mathcal{V}\left( A\right) +\ln \left( 1/\delta \right) }{%
		\lambda \left( h\right) }\right\} <\delta .
\end{equation*}%
Inserting the definition of $\lambda \left( h\right) $ completes the proof.

\subsection{PAC-Bayes bounds with data-dependent priors\label{Section Proofs
		for PAC-Bayes bounds}}

\begin{theorem}[Restatement of Theorem \protect\ref{Theorem Rivasplata}]
	Let $F:\mathcal{H}\times \mathcal{X}^{n}\rightarrow R$ be measurable. With
	probability at least $1-\delta $ in the draw of $\mathbf{X}\sim \mu ^{n}$ we
	have for all $P\in \mathcal{P}\left( \mathcal{H}\right) $%
	\begin{equation*}
		\mathbb{E}_{h\sim P}\left[ F\left( h,\mathbf{X}\right) \right] \leq KL\left(
		P,Q_{\mathbf{X}}\right) +\ln \mathbb{E}_{\mathbf{X}}\left[ \mathbb{E}_{h\sim
			Q_{\mathbf{X}}}\left[ e^{F\left( h,\mathbf{X}\right) }\right] \right] +\ln
		\left( 1/\delta \right) .
	\end{equation*}
\end{theorem}

\begin{proof}
	Let $P\in \mathcal{P}_{1}\left( \mathcal{H}\right) $ be arbitrary. Then%
	\begin{eqnarray*}
		\mathbb{E}_{h\sim P}\left[ F\left( h,\mathbf{X}\right) \right] -KL\left(
		P,Q_{\mathbf{X}}\right) &=&\ln \exp \left( \mathbb{E}_{h\sim P}\left[
		F\left( h,\mathbf{X}\right) -\ln \left( dP/dQ_{\mathbf{X}}\left( h\right)
		\right) \right] \right) \\
		&\leq &\ln \mathbb{E}_{h\sim P}\left[ \exp \left( F\left( h,\mathbf{X}%
		\right) -\ln \left( dP/dQ_{\mathbf{X}}\left( h\right) \right) \right) \right]
		\\
		&=&\ln \mathbb{E}_{h\sim P}\left[ e^{F\left( h,\mathbf{X}\right) }(dP/dQ_{%
			\mathbf{X}}\left( h\right) )^{-1}\right] \\
		&=&\ln \mathbb{E}_{h\sim Q_{\mathbf{X}}}\left[ e^{F\left( h,\mathbf{X}%
			\right) }\right] .
	\end{eqnarray*}%
	So we only need to bound the last expression, which is independent of $P$.
	But\ by Markov's inequality (Lemma \ref{Markov's inequality}) with
	probability at least $1-\delta $ in $\mathbf{X}\sim \mu ^{n}$%
	\begin{eqnarray*}
		\ln \mathbb{E}_{h\sim Q_{\mathbf{X}}}\left[ e^{F\left( h,\mathbf{X}\right) }%
		\right] &\leq &\ln \mathbb{E}_{\mathbf{X}}\left[ \exp \left( \ln \mathbb{E}%
		_{h\sim Q_{\mathbf{X}}}\left[ e^{F\left( h,\mathbf{X}\right) }\right]
		\right) \right] +\ln \left( 1/\delta \right) \\
		&=&\ln \mathbb{E}_{\mathbf{X}}\left[ \mathbb{E}_{h\sim Q_{\mathbf{X}}}\left[
		e^{F\left( h,\mathbf{X}\right) }\right] \right] +\ln \left( 1/\delta \right) 
		\text{.}
	\end{eqnarray*}
\end{proof}

We close with a method to deal with the problem of parameter optimization in
the derivation of PAC-Bayesian bounds from our results. We apply it to the
case of Gaussian priors centered on the output of stable algorithms, but
analogous results can be equally derived for Hamiltonians with bounded
differences or sub-gaussian Hamiltonians. Our first result is a PAC-Bayesian
bound analogous to part (iii) of Theorem \ref{Theorem Gaussian}.

\begin{theorem}
	\label{Theorem PAC-Bayes Appendix}Under the conditions of Theorem \ref%
	{Theorem Gaussian} we have with probability at least $1-\delta $ in $\mathbf{%
		X}\sim \mu ^{n}$ for all $P\in \mathcal{P}\left( \mathcal{H}\right) $ that%
	\begin{equation*}
		\mathbb{E}_{h\sim P}\left[ \Delta \left( h,\mathbf{X}\right) \right] >\sqrt{%
			\frac{\frac{3}{\sigma ^{2}}\mathcal{V}\left( A\right) +2KL\left( P,Q_{%
					\mathbf{X}}\right) +\ln \left( 2KL\left( P,Q_{\mathbf{X}}\right) /\delta
				\right) }{n}}.
	\end{equation*}
\end{theorem}

To prove this we first establish the following intermediate result.

\begin{proposition}
	\label{Proposition convert to Pac Bayes}Under the conditions of Theorem \ref%
	{Theorem Gaussian} let $K>0$ and $\delta >0$. Then with probability at least 
	$1-\delta $ as $\mathbf{X}\sim \mu ^{n}$ we have for any $P\in \mathcal{P}%
	\left( \mathcal{H}\right) $ with $KL\left( P,Q_{\mathbf{X}}\right) \leq K$
	that%
	\begin{equation*}
		\mathbb{E}_{h\sim P}\left[ \Delta \left( h,\mathbf{X}\right) \right] \leq 
		\sqrt{\frac{\frac{3}{\sigma ^{2}}\mathcal{V}\left( A\right) +K+\ln \left(
				1/\delta \right) }{n}}
	\end{equation*}%
	\bigskip
\end{proposition}

\begin{proof}
	If $KL\left( P,Q_{\mathbf{X}}\right) \leq K$ we get from Theorem \ref%
	{Theorem Rivasplata} and inequality (\ref{Moment bound Gaussian}) with
	probability at least $1-\delta $ in $\mathbf{X}\sim \mu ^{n}$%
	\begin{eqnarray*}
		\mathbb{E}_{h\sim P}\left[ \lambda \Delta \left( h,\mathbf{X}\right) \right]
		-K &\leq &\mathbb{E}_{h\sim P}\left[ \lambda \Delta \left( h,\mathbf{X}%
		\right) \right] -KL\left( P,Q_{\mathbf{X}}\right) \\
		&\leq &\ln \mathbb{E}_{\mathbf{X}}\left[ \mathbb{E}_{h\sim Q_{\mathbf{X}}}%
		\left[ e^{\lambda \Delta \left( h,\mathbf{X}\right) }\right] \right] +\ln
		\left( 1/\delta \right) \\
		&\leq &\frac{\lambda ^{2}}{4n}+\frac{3}{\sigma ^{2}}\mathcal{V}\left(
		A\right) +\ln \left( 1/\delta \right) .
	\end{eqnarray*}%
	Bring $K$ to the other side, divide by $\lambda $ and and optimize $\lambda $
	to complete the proof.
\end{proof}

To get rid of $K$ we use a model-selection lemma from \cite{Anthony99}.

\begin{lemma}
	\label{Lemma model selection}(Lemma 15.6 in \cite{Anthony99}) Suppose $\Pr $
	is a probability distribution and%
	\begin{equation*}
		\left\{ E\left( \alpha _{1},\alpha _{2},\delta \right) :0<\alpha _{1},\alpha
		_{2},\delta \leq 1\right\}
	\end{equation*}%
	is a set of events, such that
	
	(i) For all $0<\alpha \leq 1$ and $0<\delta \leq 1$,%
	\begin{equation*}
		\Pr \left\{ E\left( \alpha ,\alpha ,\delta \right) \right\} \leq \delta .
	\end{equation*}
	
	(ii) For all $0<\alpha _{1}\leq \alpha \leq \alpha _{2}\leq 1$ and $0<\delta
	_{1}\leq \delta \leq 1$%
	\begin{equation*}
		E\left( \alpha _{1},\alpha _{2},\delta _{1}\right) \subseteq E\left( \alpha
		,\alpha ,\delta \right) .
	\end{equation*}%
	Then for $0<a,\delta <1$,%
	\begin{equation*}
		\Pr \bigcup_{\alpha \in \left( 0,1\right] }E\left( \alpha a,\alpha ,\delta
		\alpha \left( 1-a\right) \right) \leq \delta .
	\end{equation*}
\end{lemma}

\begin{proof}
	Define the events 
	\begin{equation*}
		E\left( \alpha _{1},\alpha _{2},\delta \right) :=\left\{ \exists P,KL\left(
		P,Q_{\mathbf{X}}\right) \leq \alpha _{2}^{-1},\mathbb{E}_{h\sim P}\left[
		\Delta \left( h,\mathbf{X}\right) \right] >\sqrt{\frac{\frac{3}{\sigma ^{2}}%
				\mathcal{V}\left( A\right) +\alpha _{1}^{-1}+\ln \left( 1/\delta \right) }{n}%
		}\right\} .
	\end{equation*}%
	By Proposition \ref{Proposition convert to Pac Bayes} they satisfy (i) of
	Lemma \ref{Lemma model selection} and it is easy to see, that (ii) also
	holds. If we set $a=1/2$ the conclusion of Lemma \ref{Lemma model selection}
	becomes%
	\begin{equation*}
		\mathbb{E}_{h\sim P}\left[ \Delta \left( h,\mathbf{X}\right) \right] >\sqrt{%
			\frac{\frac{3}{\sigma ^{2}}\mathcal{V}\left( A\right) +2KL\left( P,Q_{%
					\mathbf{X}}\right) +\ln \left( 2KL\left( P,Q_{\mathbf{X}}\right) /\delta
				\right) }{n}}.
	\end{equation*}%
	\bigskip
\end{proof}

To get a PAC-Bayesian bound analogous to Theorem \ref{Theorem Gaussian} (iv)
we proceed similarly and obtain the intermediate bound that for $\delta >0$
with probability at least $1-\delta $ in $\mathbf{X}\sim \mu ^{n}$ for all $%
P $ such that $\mathbb{E}_{h\sim P}\left[ v\left( h\right) \right] \leq V$
and $KL\left( P,Q_{\mathbf{X}}\right) \leq K$%
\begin{equation*}
	\mathbb{E}_{h\sim P}\left[ \Delta \left( h,\mathbf{X}\right) \right] \leq 2%
	\sqrt{V\left( c^{2}+\frac{\frac{3}{\sigma ^{2}}\mathcal{V}\left( A\right)
			+K+\ln \left( 1/\delta \right) }{n}\right) }+b\left( c^{2}+\frac{\frac{3}{%
			\sigma ^{2}}\mathcal{V}\left( A\right) +K+\ln \left( 1/\delta \right) }{n}%
	\right) .
\end{equation*}

Then we proceed as above, except that we have to use Lemma \ref{Lemma model
	selection} twice, once with $K$ and once with $nV$. The result of this
mechanical procedure is

\begin{theorem}
	For $\delta >0$ with probability at least $1-\delta $ in $\mathbf{X}\sim \mu
	^{n}$ for all $P$%
	\begin{equation*}
		\mathbb{E}_{h\sim P}\left[ \Delta \left( h,\mathbf{X}\right) \right] \leq 2%
		\sqrt{\left( 2\mathbb{E}_{h\sim P}\left[ v\left( h\right) \right]
			+1/n\right) \frac{\frac{3}{\sigma ^{2}}\mathcal{V}\left( A\right) +C}{n}}+%
		\frac{\frac{3}{\sigma ^{2}}\mathcal{V}\left( A\right) +C}{n},
	\end{equation*}%
	where%
	\begin{equation*}
		C=2KL\left( P,Q_{\mathbf{X}}\right) +1+\ln \left( 2\left( KL\left( P,Q_{%
			\mathbf{X}}\right) +1\right) \left( 2\left( n\mathbb{E}_{h\sim P}\left[
		v\left( h\right) \right] +1\right) \right) /\delta \right) .
	\end{equation*}
\end{theorem}

\section{Table of notation\label{table of notation}}

,%
\begin{tabular}{l|l}
	\hline
	$\mathcal{X}$ & space of data \\ \hline
	$\mu $ & probability of data \\ \hline
	$n$ & sample size \\ \hline
	$\mathbf{x}$ & generic member $\left( x_{1},...,x_{n}\right) \in \mathcal{X}%
	^{n}$ \\ \hline
	$\mathbf{X}$ & training set $\mathbf{X}=\left( X_{1},...,X_{n}\right) \sim
	\mu ^{n}$ \\ \hline
	$\mathcal{H}$ & loss class (loss fctn. composed with hypotheses, $h:\mathcal{%
		X\rightarrow }\left[ 0,\infty \right) $) \\ \hline
	$\mathcal{P}\left( \mathcal{H}\right) $ & probability measures on $\mathcal{H%
	}$ \\ \hline
	$\pi $ & nonnegative a-priori measure on $\mathcal{H}$ \\ \hline
	$L\left( h\right) $ & $L\left( h\right) =\mathbb{E}_{x\sim \mu }\left[
	h\left( x\right) \right] $, expected loss of $h\in \mathcal{H}$ \\ \hline
	$\hat{L}\left( h,\mathbf{X}\right) $ & $\hat{L}\left( h,\mathbf{X}\right)
	=\left( 1/n\right) \sum_{i=1}^{n}h\left( X_{i}\right) $, empirical loss of $%
	h\in \mathcal{H}$ \\ \hline
	$\Delta \left( h,\mathbf{X}\right) $ & $L\left( h\right) -\hat{L}\left( h,%
	\mathbf{X}\right) $, generalization gap \\ \hline
	$Q$ & $Q:\mathbf{x}\in \mathcal{X}^{n}\mapsto Q_{\mathbf{x}}\in \mathcal{P}%
	\left( \mathcal{H}\right) $, stochastic algorithm \\ \hline
	$Q_{\mathbf{x}}\left( h\right) $ & density w.r.t. $\pi $ of $Q_{\mathbf{x}}$
	evaluated at $h\in \mathcal{H}$, $Q_{\mathbf{x}}\left( h\right) =\exp \left(
	H_{Q}\left( h,\mathbf{x}\right) \right) $ \\ \hline
	$H$ & $H:\mathcal{H\times X}^{n}\rightarrow \mathbb{R}$, Hamiltonian \\ 
	\hline
	$Z$ & $Z:\mathcal{X}^{n}\rightarrow \mathbb{R}$, $Z\left( \mathbf{x}\right)
	=\int_{\mathcal{H}}\exp \left( H\left( h,\mathbf{x}\right) \right) d\pi
	\left( h\right) $, partition function \\ \hline
	$H_{Q}$ & $H_{Q}\left( h,\mathbf{x}\right) =H\left( h,\mathbf{x}\right) -\ln
	Z\left( \mathbf{x}\right) =\ln Q_{\mathbf{x}}\left( h\right) $ \\ \hline
	$S_{y}^{k}$ & $S_{y}^{k}\mathbf{x}=\left(
	x_{1},...,x_{k-1},y,x_{k+1},...,x_{n}\right) $, substitution operator \\ 
	\hline
	$D_{y,y^{\prime }}^{k}$ & $\left( D_{y,y^{\prime }}^{k}f\right) \left( 
	\mathbf{x}\right) =f\left( S_{y}^{k}\mathbf{x}\right) -f\left( S_{y^{\prime
	}}^{k}\mathbf{x}\right) $, partial difference operator \\ \hline
	$kl\left( p,q\right) $ & $kl\left( p,q\right) =p\ln \frac{p}{q}+\left(
	1-p\right) \ln \frac{1-p}{1-q}$, re. entropy of Bernoulli variables \\ \hline
	$KL\left( \rho ,\nu \right) $ & $\int \left( \ln \frac{d\rho }{d\nu }\right)
	d\rho $, KL-divergence of p.-measures $\rho $ and $\nu $ \\ \hline
	$\left\Vert .\right\Vert $ & Euclidean norm on $\mathbb{R}^{D}$. \\ \hline
\end{tabular}%
\bigskip 

\end{document}